\documentclass[10pt,journal,compsoc]{IEEEtran}               % final (journal style)
\usepackage{graphicx}  % allow us to embed graphics files

\DeclareGraphicsExtensions{.pdf,.png,.jpg,.jpeg} % for pdflatex we expect .pdf, .png, or .jpg files
\graphicspath{{figures/}{pictures/}{images/}{./}} % where to search for the images

\usepackage{microtype}                 % use micro-typography (slightly more compact, better to read)
\PassOptionsToPackage{warn}{textcomp}  % to address font issues with \textrightarrow
\usepackage{textcomp}                  % use better special symbols
\usepackage{mathptmx}                  % use matching math font

\usepackage{enumitem}

\usepackage{times}                     % we use Times as the main font
\usepackage{cite}                      % needed to automatically sort the references
\usepackage{tabu}                      % only used for the table example
\usepackage{booktabs}                  % only used for the table example

\usepackage{amsmath} 
\usepackage{amssymb} 
\usepackage{amsthm}
\usepackage{bm}
\usepackage[ruled,vlined]{algorithm2e}
\usepackage{hyperref}
\usepackage{tikz}
\usepackage{pgfplots}
\usepackage[caption=false,font=normalsize,labelfont=sf,textfont=sf]{subfig}
\pgfplotsset{compat=1.17}
\usepgfplotslibrary{groupplots}
\pgfplotsset{
  every axis/.append style={
    no markers,
    grid=major,
    grid style={dashed},
    legend style={font=\tiny},
    ylabel style={font=\scriptsize},
    xlabel style={font=\scriptsize},
    width=\linewidth
  },
  every axis plot/.append style={line width=1.2pt, line join=round},
  every axis legend/.append style={legend columns=2},
  group/group size=3 by 1,
  every x tick label/.append style={alias=XTick,inner xsep=0pt},
  every x tick scale label/.style={at=(XTick.base east),anchor=base west}
}

\definecolor{col1}{RGB}{36, 61, 93}
\definecolor{col2}{RGB}{77, 143, 145}
\definecolor{col3}{RGB}{228, 206, 135}
\definecolor{col4}{RGB}{186, 161, 22}
\definecolor{col5}{RGB}{45, 3, 59}
\definecolor{col6}{RGB}{193, 71, 233}
\definecolor{col7}{RGB}{126, 134, 255}
\definecolor{col8}{RGB}{85, 170, 255}

\tikzset{
  curve1/.style={col1},
  curve2/.style={col2},
  curve3/.style={col3},
  curve4/.style={col4},
  curve5/.style={col5},
  curve6/.style={col6},
  curve7/.style={col7},
  curve8/.style={col8},
  curve9/.style={col1, densely dotted},
  curve10/.style={col2, dotted},
  curve11/.style={col3, dashdotted},
  curve12/.style={col4, dashed},
}

\newtheoremstyle{theoremstyle} % name
    {\topsep}                    % Space above
    {\topsep}                    % Space below
    {}                   % Body font
    {}                           % Indent amount
    {\bfseries}                   % Theorem head font
    { :}                          % Punctuation after theorem head
    {.5em}                       % Space after theorem head
    {}  % Theorem head spec (can be left empty, meaning ‘normal’)
\theoremstyle{theoremstyle}
\newtheorem{Pro}{Proposition}

\newcommand{\Dict}{\mathcal{D}}

\newcommand{\Lamb}{\bm{\lambda}}
\newcommand{\Bary}{\julienRevision{Y}(\Lamb , \Dict)}
\newcommand{\BaryAll}{\julienRevision{Y}(\Lamb_{n} , \Dict)}

\newcommand{\BaryLamb}{\julienRevision{Y}(\Lamb)}
\newcommand{\BaryLambO}{\julienRevision{Y}(\Lamb^{t})}
\newcommand{\BaryDict}{\julienRevision{Y}(\Dict)}
\newcommand{\BaryDictO}{\julienRevision{Y}(\Dict_{t})}
\newcommand{\wasserstein}{\julienRevision{W}}
\newcommand{\War}{\wasserstein^2}
\newcommand{\Ass}{\phi_{1}}
\newcommand{\AssM}{\phi_{m}}
\newcommand{\AssI}{\phi_{i}}

\newcommand{\keanu}[1]{\textcolor{black}{#1}}

\newcommand{\frechetEnergy}{E_F}
\newcommand{\dictionaryEnergy}{E_D}
\newcommand{\weightEnergy}{E_W}
\newcommand{\atomEnergy}{E_A}
\newcommand{\individualAtomEnergy}{e_A}

\newcommand{\julien}[1]{\textcolor{black}{#1}}
\newcommand{\julienQuestion}[1]{\textcolor{black}{#1}}
\newcommand{\julienRevision}[1]{\textcolor{black}{#1}}

\begin{document}

%% Paper title.
\title{Wasserstein Dictionaries\\of Persistence Diagrams}
% \title{Wasserstein Dictionary Learning for Persistence Diagrams.}

%% This is how authors are specified in the journal style

%% indicate IEEE Member or Student Member in form indicated below
% \author{Julien Tierny, Julie Delon, Keanu Sisouk}
\author{Keanu Sisouk, Julie Delon, Julien Tierny}
\IEEEtitleabstractindextext{%
\begin{abstract}
This paper presents a computational framework for the concise
encoding of an ensemble of persistence diagrams, in the form of  weighted
Wasserstein barycenters \cite{Turner2014, vidal_vis19}
of a dictionary of \emph{atom diagrams}. We introduce a multi-scale gradient
descent approach for the efficient resolution of the corresponding minimization
problem, which interleaves the optimization of the barycenter weights with the
optimization of the \emph{atom diagrams}. Our approach leverages the
analytic expressions for the gradient of both sub-problems to ensure fast
iterations and it additionally exploits shared-memory parallelism.
Extensive experiments on public ensembles demonstrate the efficiency of our
approach, with Wasserstein dictionary computations in the orders of minutes for
the largest examples. We show the utility of our contributions in two
applications. First, we apply Wassserstein dictionaries to \emph{data
reduction} and reliably compress persistence diagrams by
concisely representing them with their
weights
% \julien{barycentric coordinates}
in the dictionary.
Second, we present a \emph{dimensionality reduction} framework based on a
Wasserstein dictionary defined with a small number of atoms (typically three)
and encode the dictionary as a low dimensional simplex embedded in a visual
space (typically in 2D). In both applications, quantitative experiments assess
the relevance of our framework. Finally, we provide a C++ implementation that
can be used to reproduce our results.
\end{abstract}

% end of abstract-julien

%% Keywords that describe your work. Will show as 'Index Terms' in journal
%% please capitalize first letter and insert punctuation after last keyword
%\keywords{Topological data analysis, ensemble data, persistence diagrams,
%optimization.}

\begin{IEEEkeywords}
Topological data analysis, ensemble data, persistence diagrams.
% optimization.
\end{IEEEkeywords}}

\maketitle

\IEEEdisplaynontitleabstractindextext

\IEEEpeerreviewmaketitle

\section{Introduction} %for journal use above \firstsection{..} instead

\IEEEPARstart{A}{s} 
\julien{measurement}
% techniques 
\julien{devices}
and \julien{numerical techniques} 
% simulations 
are becoming more and more 
advanced, 
% resulting in more geometrically complex data sets. 
\julien{datasets are becoming more and more complex 
geometrically.}
This geometrical complexity 
% renders 
\julien{makes}
interactive exploration and analysis difficult, 
which \julien{challenges the interpretation of the data by the users.}
% hinders the users 
% in interpreting the data. 
This motivates the creation of expressive data 
abstractions, capable of encapsulating the main features of interest 
\julien{of} the 
data into simple representations, visually conveying the \julien{main} 
information to the user.

% In this setting, 
Topological Data Analysis (TDA) \cite{edelsbrunner09} \julien{is a 
family of techniques which precisely addresses this issue. It provides
concise topological descriptors of the main structural features hidden in a 
dataset. The relevance of TDA for analyzing scalar data, its efficiency and 
robustness have been documented in a number of visualization tasks 
\cite{heine16}. Examples of successful applications include }
% has provided over the last 
% two decades a panoply of generic techniques, encapsulating the main features of 
% interest in scalar data, whose efficiency and robustness has been demonstrated 
% in a number of visualization tasks \cite{heine16}. Examples of convincing 
% applications include 
turbulent combustion\cite{bremer_tvcg11, gyulassy_ev14, 
laney_vis06 }, material sciences \cite{favelier16, gyulassy_vis07,  
gyulassy_vis15, soler_ldav19}, nuclear energy \cite{beiNuclear16}, fluid 
dynamics \cite{kasten_tvcg11, nauleau_ldav22}, bioimaging 
\cite{beiBrain18, topoAngler, carr04}, chemistry \cite{harshChemistry, 
chemistry_vis14, Malgorzata19, olejniczak_pccp23} or astrophysics 
\cite{shivashankar2016felix, sousbie11}. 

\julien{Among the different topological descriptors studied in TDA (such as 
the merge and contour trees \cite{tarasov98, carr00,
MaadasamyDN12, AcharyaN15, CarrWSA16, gueunet_tpds19}, the Reeb graph 
\cite{biasotti08, pascucci07, tierny_vis09, parsa12, DoraiswamyN13, 
gueunet_egpgv19} , or the Morse-Smale complex \cite{forman98, EdelsbrunnerHZ01,
EdelsbrunnerHNP03,
BremerEHP03, Defl15, robins_pami11, ShivashankarN12, gyulassy_vis18}), the 
Persistence Diagram (\autoref{fig:gaussianPDStability}) is a particularly 
prominent example. As described in \autoref{section:distDgm}, it is a concise 
topological descriptor which captures the main structural features in a dataset 
and which assesses their individual importance.}

\julien{In addition to the challenge of increased geometrical complexity 
(discussed above), a new difficulty has recently emerged in many applications, 
with the notion of \emph{ensemble dataset}. These representations describe a 
given phenomenon not only with a single dataset, but with a \emph{collection} 
of datasets, called \emph{ensemble members}. In that context, the topological 
analysis of an ensemble dataset consequently results in an ensemble of 
corresponding topological descriptors (e.g. one persistence diagram per 
ensemble member).}

% Then, a major challenge consists in developping statistical tools for such an
% ensemble of topological descritpors, to facilitate its interactive analysis
% and visualization. For this, a research question deals with the definition of
% a concise, yet informative, representation of the ensemble of
% topological descriptors. This would enable its compression or its
% direct visualization via planar layout (where each point represents a descriptor
% and where distances between points are faithful with regard to the intrinsinc differences
% between the descriptors).
% At a technical level, the fundamendal problem behind this compact representation
% formulation is the understanding of the

\julienRevision{Then, a major challenge consists in developing practical tools
for such an ensemble of topological descriptors, to facilitate its processing,
analysis and visualization. Such tools include compression approaches (to
facilitate the manipulation of the ensemble of descriptors) or
visualization methods (for instance, with planar layouts, where each point
encodes a descriptor and the distance between a pair of points encodes the
intrinsic differences between the corresponding descriptors).}

\julienRevision{To enable the above tools,
% applications,
a key research question deals with the definition of a concise,
yet informative, encoding  of the ensemble of descriptors.
A promising research direction consists in defining a
\emph{dictionary} (i.e. a set of reference descriptors, or \emph{atoms}), such
that the topological descriptors of the ensemble can be concisely encoded
by expressing them as a specific \emph{function} of the atoms (e.g.
a linear combination).
% More precisely,
At a technical level,
% thi
this requires to
% understand and
accurately
capture and model
% the difficulty behind the definition of such
% a compact representation is the understanding of
the implicit relations (i.e. the possible functions) which link the
different
% topological
descriptors of the ensemble.}
% \julien{Then, a major challenge consists in developing statistical
% tools for such an ensemble of topological descriptors, to support its
% interactive analysis and interpretation. Specifically, a key aspect in this
% challenge deals with the understanding of the implicit relations which link the
% different descriptors of the ensemble.

A series of recent works
% explored this
% direction,
\julienRevision{started the exploration of this overall direction,}
in particular with the notion of \emph{average topological
representation} \cite{Turner2014, lacombe2018, vidal_vis19, YanWMGW20, 
pont_vis21}.
These techniques can produce a topological descriptor which nicely summarizes
the ensemble. However, they do not capture the implicit relations between the 
different topological descriptors.

% One abstraction data this paper focuses 
% on is the persistence diagram \cite{edelsbrunner09, EdelsbrunnerHZ01}. The 
% Persistent Homology theory enables translation of specific-field representation 
% of main features into topological terms, and distinguishes noise generated by 
% providing importance measures. Thus Persistent Homology is one way to reduce the 
% needed space in a hard drive to store the main features of a data. In modern 
% days, one may want to reduce furthermore the data stored in hard drive as 
% ensembles of persistence diagrams can still be heavy. Moreover, persistence 
% diagrams are still infinite dimensional objects making their ambient space 
% difficult to study and visualize.

\julien{This paper addresses this issue by introducing a simple and efficient 
approach for the estimation of linear relations
between persistence diagrams on their associated Wasserstein metric space.
% within an ensemble of
% persistence diagrams.
% Specifically
Inspired by previous work on histograms \cite{schmitz2018wasserstein}, our
approach provides a
linear encoding of the input ensemble, where each  diagram is represented
% in the form of 
as a 
weighted Wasserstein barycenter \cite{Turner2014, vidal_vis19} of a 
\emph{dictionary} of automatically optimized diagrams called \emph{atom 
diagrams}. 
We introduce a novel multi-scale gradient descent algorithm 
(\autoref{sec_algorithm}) for the efficient resolution of the corresponding 
minimization problem (\autoref{sec:dictEnc}),
for
which
we interleave the optimization of the barycenter
weights (\autoref{subsect:weight}) with the optimization of the atom diagrams 
(\autoref{subsect:atom}). 
Extensive experiments (\autoref{sec_results}) on public ensembles demonstrate 
the efficiency of our
approach, with Wasserstein dictionary computations in the orders of minutes for
the largest examples.
We illustrate the relevance of our contributions for 
the visual analysis of ensemble data with two applications, data reduction 
(\autoref{sec: dataRed}) and dimensionality reduction 
(\autoref{sec_dimensionalityReduction}).}

% This paper addresses those problems by providing an algorithm for the 
% approximation of a set of persistence diagrams using Wasserstein barycenters of 
% a few persistence diagrams, called \textit{atoms}, contained in a dictionary. 
% Our approach is optimizing the dictionary and barycentric weights used in the 
% construction of the barycenters using a classic gradient descent, with 
% heuristics motivated by practical observations. One particular aspect of our 
% algorithm is using a progressive approach similar to Vidal et al. 
% \cite{vidal_vis19}. This progressivity aspect allows our approach to focus more 
% on the main features in the original ensemble, thus having a good encoding of 
% them. This dictionary and the barycentric weights would then be the only objects 
% stored in the hard drive for further studies on the generated barycenters, this 
% ensures the data reduction. Experiments on synthetic and real life data showed 
% promising result for our algorithm. Firstly, we want to show that the clustering 
% result is the same between the input diagrams and their approximations. This 
% would demonstrate the relevance of our data reduction scheme. Secondly our 
% framework can also provide a projection of the input diagrams on the 2D space 
% and enable exploration and visual analysis. Our method is parallelizable, giving 
% us additional speedups on workstation.

\subsection{Related Work}

The literature related to our work can be classified into three main classes:
\textit{(i)} uncertainty visualization, \textit{(ii)} ensemble visualization,
and \textit{(iii)} \julien{topological methods for ensembles.}
% dictionary of persistence diagrams.

\noindent \textbf{\emph{(i)} Uncertainty visualization:}
\julien{Data variability can be represented in the form of \emph{uncertain}
datasets, by considering the data at each point of the domain as
a random variable, associated with an explicit probability density function
(PDF).}
% Data can have variability, in
% this context \textit{uncertain} data encode this variability by considering each
% point of the domain as a random variable, which has an explicit probability
% density function (PDF).
\julien{The analysis and visualization of uncertain data has been recognized as
a major challenge in the visualization community}
% Analyzing uncertain data is an infamous challenge in the
% visualization community
\cite{uncertainty1, uncertainty_survey2, uncertainty2,
uncertainty3, uncertainty4, uncertainty_survey1}. Several techniques
% have seen the day,
\julien{have been proposed}
either dealing with the entropy of the random variables
\cite{uncertainty_entropy},
% on
\julien{or}
their correlation \cite{uncertainty_correlation}
or gradient variation \cite{uncertainty_gradient}.
\julien{The effect of data uncertainty on feature extraction has also been
studied (for instance for level set extraction}
% , several approaches dealt with the uncertain extraction
% of level sets}
% on various feature
% extraction approaches has been studied, for instance for isosurface
% construction}
% This uncertainty is reflected
% on the position of level sets and critical points when geometrical constructions
% are extracted from uncertain data
\cite{uncertainty_isosurface1,
uncertainty_isosurface2, uncertainty_isosurface3, uncertainty_isocontour1,
uncertainty_nonparam, uncertainty_isosurface4, uncertainty_interp}),
\julien{for various interpolation schemes and PDF models (e.g. Gaussian
\cite{liebmann1, otto1, otto2, petz} or uniform \cite{bhatia, gunther, szymczak}
 distributions).}
% which have
% been studied for several interpolation schemes and PDF for the former, and for
% Gaussian \cite{liebmann1, otto1, otto2, petz} or uniform distributions
% \cite{bhatia, gunther, szymczak} for the latter.
In general, \julien{a central limitation of existing methods for uncertain data
is their design dependence on a specific PDF model (Gaussian, uniform, etc).
This challenges their usability for ensemble data, where the PDFs estimated
from the ensemble can follow an arbitrary, unknown model. Moreover, most of
these techniques do not consider multi-modal PDFs, which are however essential
when multiple trends appear in the ensemble.}
% the main
% limitations of visual methods for uncertain data are the specifics probability
% distributions (Gaussian, uniform, etc.). This questions their usability on
% ensemble data, where PDF estimated from the ensemble can follow an unknown law
% of probability. Plus most of those techniques do not consider multi-modal PDF
% models, which are essentials when several distinct trends appear in the
% ensemble.

\noindent \textbf{\emph{(ii)} Ensemble visualization:}
\julien{Another approach to model data
variability consists in using ensemble datasets. In this context,}
% One approach to study this
% variability is by considering ensemble datasets. In this framework,
the
variability is encoded
by a sequence of empirical observations (\textit{i.e} the
members of the ensemble). Established techniques typically compute geometrical 
objects, such as level sets or \julien{streamlines}, \julien{thereby} capturing
the main features for each
member of the ensemble. From there, a \textit{representative} of the
\julien{resulting} ensemble of
geometrical objects \julien{can be} computed. For this task, a few methods have
been
% established.
\julien{introduced.}
For instance spaghetti plots \cite{spaghettiPlot} are used in the
case of level-set variability, more particularly for weather data 
\cite{Potter2009, Sanyal2010}\julien{, and} box-plots \cite{whitaker2013,
Mirzargar2014} for
the variability of contours and curves. In the case of trend variability, Hummel 
et al. \cite{Hummel2013} conceived a Lagrangian framework for classification 
purposes in flow ensembles. More specifically, clustering techniques have been 
used to identify the main trends in ensemble of streamlines \cite{Ferstl2016} 
and isocontours \cite{Ferstl2016b}. However, only few techniques have applied 
this strategy to topological objects. Favelier et al. \cite{favelier_vis18} and 
Athawale et al. \cite{athawale_tvcg19} respectively introduced techniques to 
analyze
\julien{the geometrical variability of}
critical points and gradient separatrices. Overlap-based heuristics have
been studied for estimating a representative contour tree from an ensemble 
\cite{Kraus2010VisualizationOU, Wu2013ACT}. 
\julien{In the context of ensembles of histograms, Schmitz et al.
\cite{schmitz2018wasserstein} introduced a dictionary encoding approach based
on optimal transport \cite{cuturi2013}. However, this method is not directly
applicable to persistence diagrams. It focuses on a fundamentally different
object (histograms). Thus, the employed distances, geodesics and barycenters are
defined differently (in particular in an entropic form \cite{cuturi2013,
cuturi2014}) and the algorithms for their computations are drastically
different (based on Sinkhorn matrix scaling \cite{Sinkhorn}). In contrast, our
work focuses on \emph{Persistence diagrams} (\autoref{section:distDgm}), whose
associated metric space is also inspired from
optimal transport, but with various formal and computational specificities
(\autoref{sec_metricSpace}). Moreover, our approach is based on gradient descent
which, from our experience, provides better practical convergence for this kind
of problems
% with regard
% to
than
quasi-Newton techniques. Finally, we contribute a multi-scale progressive
optimization algorithm, which provides improved solutions in comparison to a
naive optimization.}

% %% JULIEN: not here.
% Recently a new method called Princiapel Geodesics Analysis (PGA), using
% principal geodesics as orthogonal basis, was introduced by Pont et al.
% \cite{pont2022principal} to approximate a whole ensemble of merge trees and
% persistence diagrams using representatives. However none of the above techniques
% used a dictionary of persistence diagrams to generate barycenters to approximate
% of an ensemble of persistence diagrams.
% % JULIEN: end of comment

\noindent \textbf{\emph{(iii)} Topological methods for ensembles:}
\julien{To analyze the relations between the persistence diagrams of an
ensemble, several key low level notions are required, such as the notion of
distance and barycenters between diagrams, for which we review the literature
here. Inspired by optimal transport \cite{Kantorovich, monge81}, the
\emph{Wasserstein} distance between persistence diagrams \cite{edelsbrunner09}
(\autoref{sec_metricSpace}) has been extensively studied
\cite{CohenSteinerEH05, Cohen-Steiner2010}. It relies on a bipartite assignment
problem, for which exact \cite{Munkres1957} and approximate \cite{Bertsekas81,
Kerber2016} implementations are available in open-source \cite{ttk17}. Based on
this distance, several approaches have explored the possibility to define a
\emph{representative} diagram of an ensemble of persistence diagrams, with the
notion of \emph{Wasserstein} barycenter. Turner et al. \cite{Turner2014}
introduced the first approach for the computation of such a barycenter.
Lacombe et al. \cite{lacombe2018} presented an approach based on entropic
transport \cite{cuturi2013, cuturi2014}. However, it requires a
pre-vectorization step which is subject to several parameters, and which is not
conducive to visualization tasks (features can no longer be individually tracked
beyond the pre-vectorization step). In contrast, Vidal
et al. \cite{vidal_vis19} introduced a vectorization-free approach which
% does
% not require a pre-vectorization and
maintains the feature assignments
explicitly. % Specifically, t
% This method
It is based on
% extended this work with
a progressive
scheme, which greatly accelerates computation in practice. These
% tools and
concepts have been recently investigated for other topological descriptors,
such as merge trees \cite{YanWMGW20, pont_vis21}. Recently,
\julienRevision{several authors have investigated another compact representation of
ensembles of topological descriptors, via a basis of representative descriptors.
For instance, Li et al. \cite{LI2023}
introduce a vectorization for merge trees, which was subsequently used by matrix
sketching procedures \cite{woodRuff2014} to create a basis of representative
merge trees. In contrast, our work focuses on persistence diagrams (which can
encode different features). Also, it directly operates on the Wasserstein metric
space of persistence diagrams,
thereby avoiding
the typical technical
difficulties associated with vectorizations (e.g. quantization and/or linearization
artifacts, potential stability issues, possible inaccuracies in vectorization reversal, etc.).}
Pont et al.
\cite{pont2022principal} introduced the notion of principal geodesic analysis
of merge trees (and persistence diagrams), with the same overall goal of
characterizing the relations between the topological descriptors of an ensemble.
In this work, we introduce a different formulation of the problem, which is both
simpler (based on the construction of weighted Wasserstein barycenters) and more
flexible (our optimization is not subject to complicated constraints such as
geodesic orthogonality). This results in a simpler implementation and
slightly faster
computations (\autoref{sec_results}).}

\subsection{Contributions}

This paper makes the following new contributions:
\begin{enumerate}

\item \julien{\emph{A simple approach for the linear encoding of 
Persistence Diagrams:} We formulate the linear encoding of an ensemble of 
persistence diagrams on their associated Wasserstein metric space as a
dictionary optimization
% problem
(\autoref{sec:dictEnc}), which simply optimizes, simultaneously, \emph{(i)}
the barycentric weights
(\autoref{subsect:weight}) and \emph{(ii)} the atoms of the dictionary 
(\autoref{subsect:atom}).}

\item \julien{ \emph{A multi-scale algorithm for the computation of a
Wasserstein dictionary of
Persistence Diagrams:}
% dictionary optimization:}
We introduce a novel, efficient algorithm for the 
optimization of the above dictionary encoding problem. 
Our algorithm leverages the analytic expressions of the gradient of both 
of the above sub-problems, to ensure fast iterations. Moreover, in comparison 
to a naive optimization, our algorithm reaches solutions of improved energy
thanks
to a multi-scale
% ,
% progressive
strategy. Finally, we leverage shared-memory
parallelism to further improve
% practical
performances.}

\item \julien{ \emph{An application to data reduction:} We present an 
application to data reduction (\autoref{sec: dataRed}), where the persistence 
diagrams of the input ensemble are significantly compressed, by solely storing 
their barycentric weights as well as the atom diagrams.}
% We demonstrate the
% relevance of this reduction with an application to the topological clustering
% of ensemble members.}

\item \julien{\emph{An application to dimensionality reduction:} We present an 
application to dimensionality reduction 
(\autoref{sec_dimensionalityReduction}), by embedding each input diagram as a 
point within a 2D view, based on its weights relative to a 
Wasserstein dictionary composed of three atoms (thereby defining a 2-simplex).}

\item \julien{\emph{Implementation:} We provide a C++ implementation of 
our algorithms that can be used for reproducibility purposes.}
% \item \textit{Wasserstein dictionary encoding for persistence diagrams}: This method revisit a Wasserstein dictionary encoding algorithm used on images, compared to this foundation, we deal with objects whose distributions permit only a one to one matching in the Wasserstein distance; furthermore, in our case the algorithm manipulate the raw persistence diagrams in contrast to the previous work where it vectorizes the inputs. This point focuses on optimizing the dictionary composed of a few diagrams and the barycentric weights. The returned dictionary and weights are explicit, and are then used to generate approximation of the inputs to study on in place of the original diagrams. The progressive approach generally provides an acceleration on the computation time of the algorithm. Our method is parallelized with respect to the number of inputs diagrams which provides additional speedups on workstation. The overall optimization problem being not convex we present heuristic conditions, based on observations, stopping the algorithm. 
% 
% \item \textit{A method for projecting persistence diagrams into the 2D space}: If the dictionary has up to 3 diagrams in it, the method enables visual ensemble analysis which cannot be done normally as the ambient space is infinitely dimensional.
% 
% \item \textit{Implementation}: We provide a lightweight C++ implementation of our algorithm that can be used for reproduction purposes.

\end{enumerate}

%This template is for papers of VGTC-sponsored conferences such as IEEE VIS, IEEE VR, and ISMAR which are published as special issues of TVCG. The template does not contain the respective dates of the conference/journal issue, these will be entered by IEEE as part of the publication production process. Therefore, \textbf{please leave the copyright statement at the bottom-left of this first page untouched}.

\section{Preliminaries}

This section presents the theoretical elements needed for the formalization of
our work.
% Those elements are adapted in Topology ToolKit \cite{ttk17}.
We
introduce the topological data representation \julien{that we} use - the
persistence diagram (\autoref{section:distDgm}) - and
\julien{its associated metric}
% the metric in their ambient space
(\autoref{sec_metricSpace}). Then we define \julien{the notion of Wasserstein
barycenter of persistence diagrams (\autoref{sec:bary}), which is a core
component of our approach (\autoref{sec:dictEnc}).}

% what a barycenter of persistence
% diagrams is (\autoref{sec:bary}), which will be crucial for our framework
% (\autoref{sec:dictEnc}).

\subsection{Persistence diagrams}
\label{section:distDgm}

\begin{figure}
 \centering % avoid the use of \begin{center}...\end{center} and use \centering instead (more compact)
 \includegraphics[width=\columnwidth]{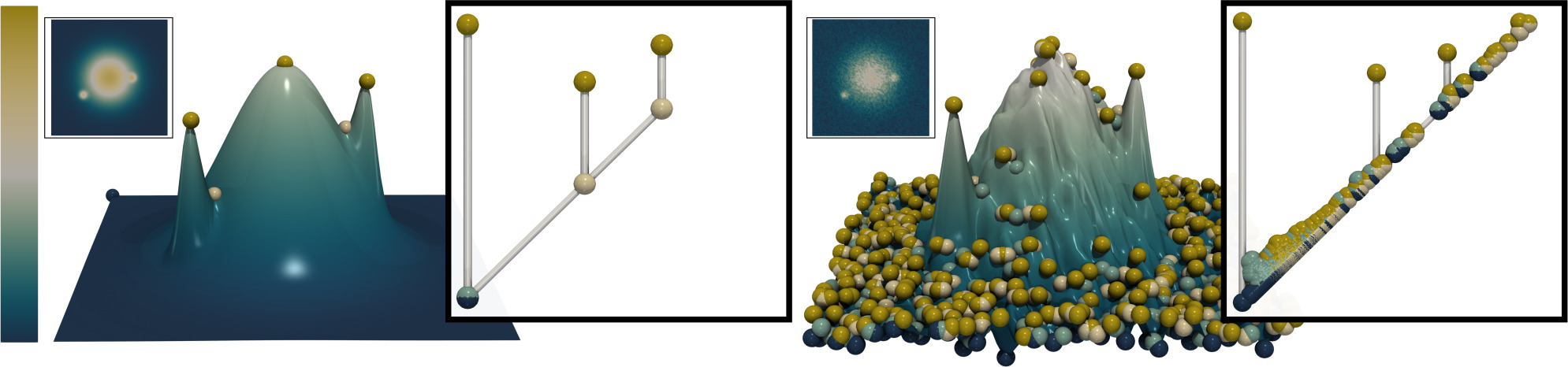}
 \caption{Persistence diagrams of a clean (left) and noisy (right)
 \julien{terrain}
%  2D Gaussian mixtures and the critical points
 (dark blue spheres: minima, dark
yellow: maxima, other: saddles).
% From left to right we can see the 2D mixture,
% then the terrain representation of the scalar field, and finally the persistence
% diagram.
The three \julien{main hills} are clearly represented \julien{with} long
% persistent
bars in the persistence \julien{diagrams}. In the noisy persistence diagram,
small bars % are associated to
\julien{encode} noise.}
 \label{fig:gaussianPDStability}
\end{figure}

\begin{figure}
 \centering % avoid the use of \begin{center}...\end{center} and use \centering
% instead (more compact)
 \includegraphics[width=\columnwidth]{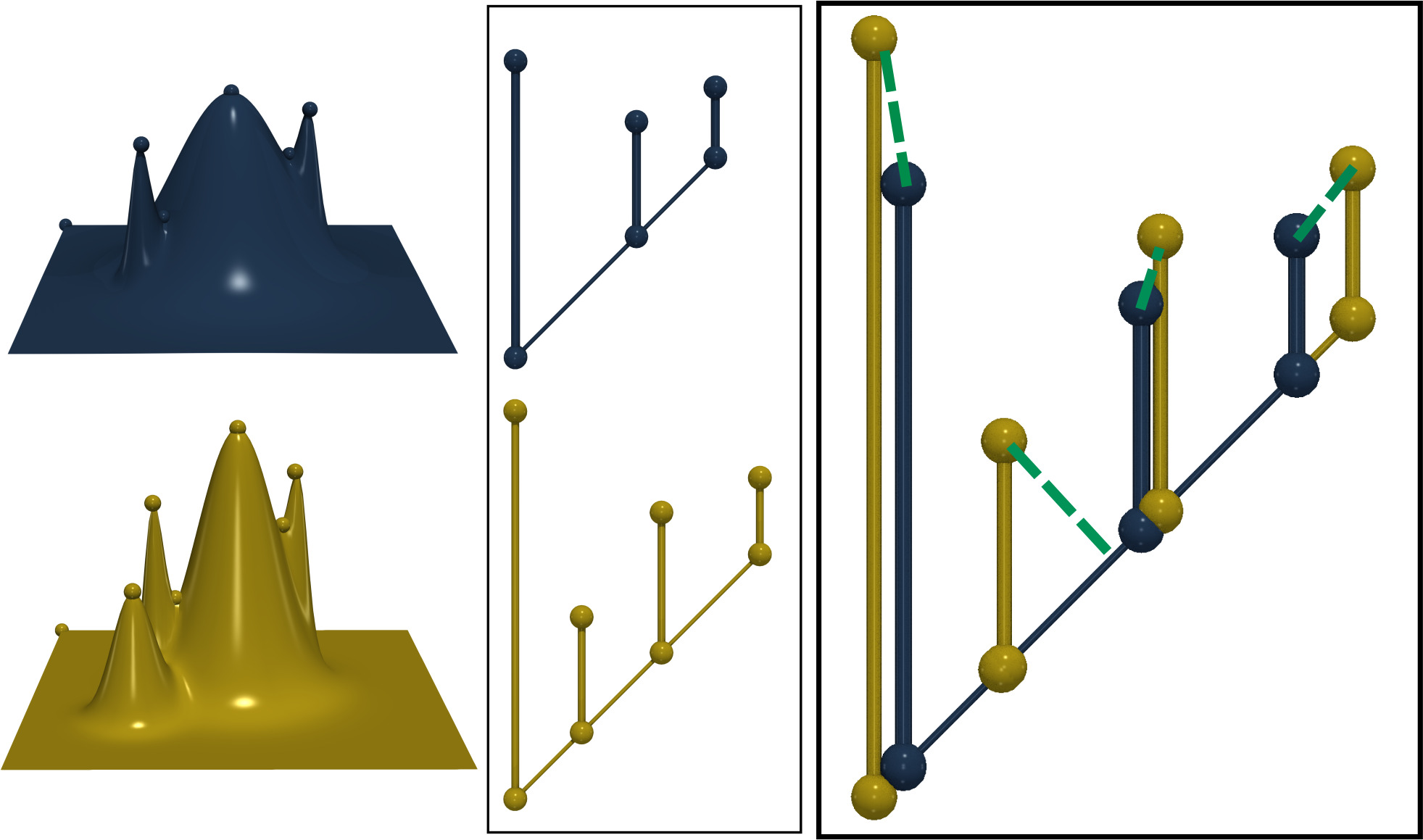}
 \caption{\julien{Optimal matching (green dashes, right) with regard to $\wasserstein$
between
the two persistence diagrams (center) of two terrains (left).}}
%  Illustration of matchings involved in the Wasserstein distance between two
% persistence diagrams. On the left we have the terrain representation of the
% scalar fields and their respective diagrams (blue and yellow); the dotted
% lines
% (green) represents the optimal matching for $W_2$ (bijection) of the pairs
% between the two persistence diagrams and as we can see on the right, a pair
% from
% one persistence diagram can be matched to the diagonal in the other one.}
 \label{fig:matchings}
\end{figure}

\julien{Each input ensemble member is given in the form of}
% The input data for computing a persistence diagram is
a piecewise linear (PL)
scalar field $f: \mathcal{M} \rightarrow \mathbb{R}$ defined on a PL
$(d_\mathcal{M})$-manifold $\mathcal{M}$, with $d_\mathcal{M}\leq 3$ for our
applications.
% For
\julien{Given an isovalue}
$w\in\mathbb{R}$, we denote $f_{-\infty}^{-1}(w) = f^{-1}\big((-\infty,
w]\big)$ the sub-level set of $f$ at $w$.
% And a
As $w$ increases, the topology of
$f_{-\infty}^{-1}$ changes at specific points of $\mathcal{M}$\julien{, called
\emph{``critical points''}}.
% In the frame of
% persistent homology $f$ is enforced to be a Morse function, meaning it  only has
% isolated and non-degenerated critical points.
Critical points are classified by
their index $\mathcal{I}$: 0 for minima, 1 for 1-saddles,
\julienRevision{$d_\mathcal{M} - 1$} for
\julienRevision{$(d_\mathcal{M}-1)$}-saddles, and
\julienRevision{$d_\mathcal{M}$} for
maxima \julien{(in practice, $f$ is enforced to
contain only isolated, non-degenerate critical points \cite{edelsbrunner90,
EdelsbrunnerHZ01})}.
\julien{According to the Elder rule \cite{edelsbrunner09}, each}
% Each
topological feature of
$f_{-\infty}^{-1}(w)$
\julien{(e.g. a connected component, a cycle, a void)}
can be associated with a pair of critical points $(c,
c')$ \julien{(with $f(c) < f(c')$ and $\mathcal{I}_c = \mathcal{I}_{c'}-1$)},
corresponding to its \textit{birth} and \textit{death} \julien{during the sweep
of the data by $w$ (from $-\infty$ to $+\infty$)}.
\julien{Such a pair $(c, c')$ is called a \emph{persistence pair}.}
% The particular point of
% persistent homology, is the foundation: the Elder rule \cite{edelsbrunner09}.
% This rule
% states that critical points can be arranged according to this
% observation in a set of pairs, such that each critical point appears in only one
% pair $(c,c')$ such that $f(c) < f(c')$ and $\mathcal{I}_c = \mathcal{I}_{c'}-1$.
\julien{For instance, when}
% Informally, it means that if
two connected components of $f_{-\infty}^{-1}(w)$
meet at a critical point $c'$, the younger one (created last) \textit{dies},
letting the oldest one (created first)
\julien{survive}.
% live.
\julien{Then, the critical points}
% Then those critical points
are
\julien{represented} visually as 2D bar codes
\julien{where the horizontal axis encodes the \emph{birth} of a feature
$\big($noted $b = f(c)\big)$ and where the vertical axis encodes its lifespan
and \emph{death} $\big($noted $d = f(c')\big)$. This representation is called
the \emph{Persistence Diagram}, noted $X$.
% in the following.
In the diagram,}
% $\big(b = f(c), d = f(c^{'})\big)$
% in a
% \textit{persistence diagram}, noted $X\ =\ X(f)$, where
salient features stand
out from the diagonal and \julien{small-amplitude} noise
% generally are
\julien{is typically}
located near the diagonal, as
\julien{shown}
% we can see
in \autoref{fig:gaussianPDStability}.
% As a summary,
\julien{In the remainder, we enumerate the points of $X$ with indices such that
$X = \{x^1,\hdots,x^K\}$ and we note $i_X = \{1,\hdots,K\}$ the set of indices
\julienRevision{(i.e. the set of all integers going from $1$ to $K$)
}.
}

% \julien{In summary,}
% a persistence
% diagram is a union of a discrete set of points in $\mathbb{R}^2,\ (b,d)$ such
% that $b > d$ and of the diagonal $\Delta = \{ (b,d) \in \mathbb{R}^2 |\ b\ =\ d
% \}$.
% \julien{In practice, to model the diagonal $\Delta$, $X$ is augmented with the }
% However for practical reasons, instead of $\Delta$ we only consider $P$ the
% orthogonal projections $(\frac{b+d}{2}, \frac{b+d}{2})$ of the off-diagonal
% points $(b,d)$. For the next parts, we enumerate the points (whether off or on
% the diagonal) of a persistence diagram $X$ such that $X = (x^1,\hdots,x^K)$ and
% we denote $i_X$ the set of indices.

\begin{figure}
 \centering % avoid the use of \begin{center}...\end{center} and use \centering
% instead (more compact)
 \includegraphics[width=\columnwidth]{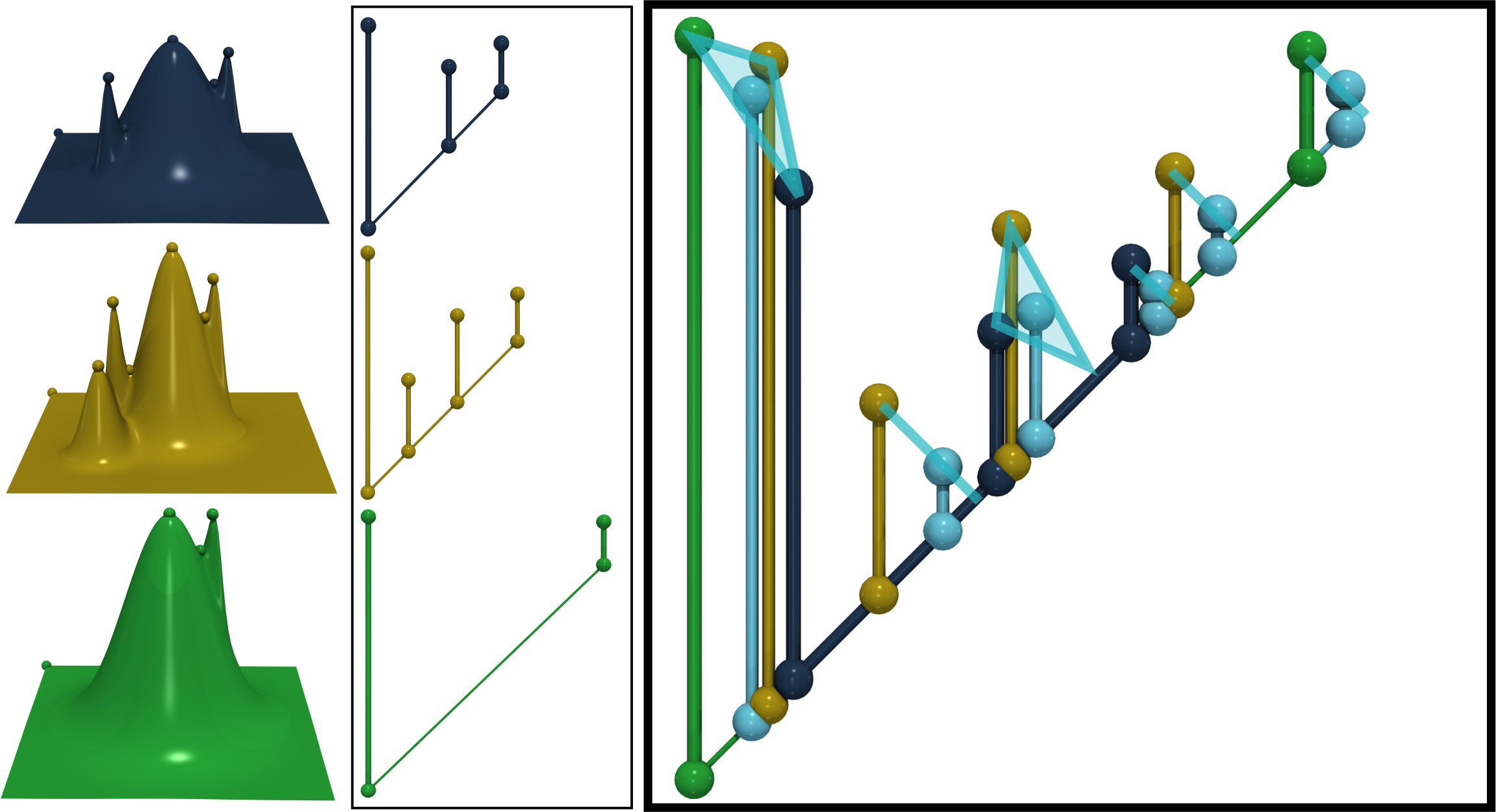}
 \caption{\julienRevision{Wasserstein barycenter (cyan, uniform weights) of $3$
persistence diagrams (center) of
$3$ terrains (left).
% for uniform barycentric weights, i.e.
% with uniform weights.
% $\lambda_1 = \lambda_2 = \lambda_3 = 1/3$.
% In the barycenter, e
Each
% Off-diagonal
barycenter point
% point in the barycenter
(cyan sphere) is the barycenter of
its matched points
% the points it matches to
in the inputs
% diagrams
(cyan triangle).}}
%  Simple illustration of a barycenter (cyan) of three persistence diagrams
% (dark
% blue, dark yellow, and green), corresponding to three scalar field (same code
% color), using uniform barycentric weights. We can see that each pair in the
% barycenter is a barycenter itself of three pairs, matched together to form a
% triangle, and cannot go outside of said triangle. Notice that pairs can still
% be
% matched to the diagonal to form a triangle (degenerated or not) and compute a
% barycenter of it.}
 \label{fig:barycenter}
\end{figure}

\subsection{Wasserstein distance}
\label{sec_metricSpace}
\julien{To evaluate the distance between two diagrams, a typical pre-processing
step consists in augmenting each diagram,}
% In
% practice, when computing the distance between two diagrams, we
% \julien{start by augmenting each diagram,}
% augment each
by
including the \julien{diagonal} projection of all \julien{the} off-diagonal
points of the other \julien{diagram}.
% on the diagonal.
To illustrate that, let us consider $X_1 =
\{x_1^1,\hdots,x_1^{K_1}\}$ and $X_2 = \{x_2^1,\hdots,x_2^{K_2}\}$.
\julien{Given an off-diagonal point $x$ (i.e. $b < d$), let
$\Delta x$ be its diagonal projection, specifically:
$\Delta x = (\frac{b+d}{2}, \frac{b+d}{2})$.
Let}
% We introduce
$P_1$ and $P_2$ \julien{be} the sets of the
% orthogonal
\julien{diagonal}
projections of the
points
% in
\julien{of}
$X_1$ and $X_2$ respectively.
\julien{Then, $X_1$ and $X_2$ are augmented into $X_1{'}$ and $X_2{'}$}
% We augment $X_1$ and $X_2$
by
considering $X_1{'} = X_1 \cup P_2$ and $X_2{'} = X_2 \cup P_1$.
% , this
\julien{This}
ensures
that $|X_1^{'}| = |X_2^{'}| = K$
\julien{(which eases distance evaluation).
% We consider in the
% following that diagrams are always augmented this way.
Specifically, we
consider in the remainder that the notations $X_1$ and $X_2$ refer to
\emph{augmented diagrams} (i.e. $|X_1| = |X_2| = K$).}
% and introduce $K$ as the size of both augmented
% diagrams.

\julien{Then, given two persistence diagrams $X_1$ and $X_2$}
% (augmented with the
% above pre-processing), }
% In this paper, we will focus mainly on
% \julien{their}
the $L^2$-Wasserstein distance \julien{between them is defined as:}
% two persistence diagrams $X_1$ and $X_2$:
\begin{equation}
% \nonumber
\label{wassDist}
\wasserstein(X_1,X_2) = \min_{\psi : i_{X} \xrightarrow{bij} i_{X}}
\sqrt{\sum\limits_{j = 1}^{K} c(x_1^j,x_2^{\psi(j)})},
\end{equation}
where $\psi$, the matching, is a bijection
\julienRevision{of the index set $i_X$
% of indices going from $1$ to $K$
% %
% indexed set
towards itself (i.e. $\psi$ is  a permutation of $i_X$).}
% \julien{(i.e. a permutation of the
% index set $i_X \julienRevision{ = \{1, 2, \dots, K\}}$).}
% \julien{between the index sets}
% $i_{X_1}$ and $i_{X_2}$}.
This bijection matches one
\julien{persistence} pair
\julien{$x_1$} of $X_1$ \julienRevision{(i.e. a pair of critical points of index
$\mathcal{I}$ and $\mathcal{I}+1$ respectively)}
% with
\julien{to}
one \julien{persistence} pair $x_2$ in $X_2$ \julienRevision{(another pair of critical
points, with the \emph{same} indices $\mathcal{I}$ and $\mathcal{I}+1$)}  whether $x_1$
and $x_2$ are on the diagonal or not
% , as
% illustrated
% \julien{shown}
% in
(\autoref{fig:matchings}).
% According to
\julien{Given} the cost $c$ in the definition of $\wasserstein$, $\psi$ is the optimal
way to transport
% the
% mass of
% \julien{persistence pairs from}
$X_1$ onto $X_2$.
\julien{In our work, we consider the cost $c(x, y) = 0$ when $x$ and $y$ are
both diagonal points, and}
% In this setting
$c(x,y) = \|x - y\|^2$ \julien{otherwise ($\|x - y\|$ denotes the
Euclidean distance
% norm
between $x$ and $y$ in the birth/death space).}
% when $x$ or $y$ is off-diagonal poin}
% $= \|x - y\|_2^2$,
% if either $x$ or $y$ is an off-diagonal point, and $c(x,y) =
% 0$ if both of them is on the diagonal.

\subsection{Wasserstein barycenter}
% of persistence diagrams}
\label{sec:bary}

\julien{Given a set of persistence diagrams $\Dict = \{a_1 ,\ldots , a_m\}$
(which we will call in the remainder \emph{dictionary}),
% the
\julienRevision{a}
Wasserstein
barycenter (\autoref{fig:barycenter}) -- or Fr\'echet mean --}
% A
% Fr\'echet mean (or barycenter)
of the dictionary $\Dict$
% a dictionary
% of persistence diagrams $\Dict = (a_1 ,\ldots , a_m)$
with
barycentric
weights
% \julien{coordinates}
$\Lamb = (\lambda_1 ,\ldots , \lambda_m)$ is a diagram,
\julien{which we note  $\Bary$ in the following}, which minimizes the
Fr\'echet energy \julienRevision{$\frechetEnergy(B)$}:
\begin{equation}
\label{eqnBaryDef}
% B^{*} \ \min_{B \in \Diag}
\nonumber
\julien{\julienRevision{\frechetEnergy(B)} = \sum\limits_{i = 1}^{m} \lambda_i \War(a_i,
B).}
\end{equation}
% such that:
% \begin{equation} \label{eqnBaryDef}
% B^{*} \in \argmin_{B \in \Diag}  \sum\limits_{i = 1}^{m} \lambda_i \War(a_i, B),
% \end{equation}
% where $\Diag$ is the set of all persistence diagrams.
$\Lamb$ is such that $\lambda_i \geq 0$ and $\displaystyle\sum\limits_{i=1}^{m} 
\lambda_i = 1$. \julien{We} denote $\Sigma_m$ the simplex of such vectors.
\julien{Intuitively,
$\Bary$ is a diagram which minimizes the above linear
combination,
given $\Lamb$,
% according to $\Lamb$,
of its squared Wasserstein distances to the
diagrams
of the dictionary $\Dict$.}

\julienRevision{The computation of the barycenter $\Bary$ requires generalizing
the pairwise augmentation described in \autoref{sec_metricSpace}.
% in the
% case of distance computations.
Specifically, each non-diagonal point of each dictionary
diagram $a_i$
% of the dictionary $\Dict$
is projected to the diagonal of all
the other dictionary diagrams $a_j$ (with $i \neq j$). After this first augmentation, each
dictionary diagram $a_i$ contains $K = \sum_{i = 1}^{m}|a_i|$ points (where $|a_i|$ is the
number of non-diagonal points in $a_i$). Then $\Bary$ is typically initialized
on the dictionary diagram $a_*$ which initially minimizes the Fr\'echet energy
$\frechetEnergy$. Let $|\Bary| = |a_*|$ be the number of non-diagonal points
of $a_*$. Then, all the non-diagonal points of all the atoms are projected
on the diagonal of $\Bary$, and reciprocally, all the non-diagonal points of
$\Bary$ are projected on the diagonal of each atom. Thus, at this stage, after this
second augmentation, each dictionary diagram $a_i$ and the candidate barycenter
$\Bary$ contains $K = \sum_{i = 1}^{m}|a_i| + |\Bary|$ points (mostly
% being
on the diagonal).}
%
%
% $\Bary$.
% Similarly, each diagram $a_i$ of $\Dict$ are augmented by inserting on its
% diagonal the projections of the non-
% Note that, prior to this barycenter computation, an augmentation
% step is required
% in a pre-process this barycenter computation has completed,}
% the size of $\Bary$, noted $|\Bary|$, is
% % at most
% equal to $\sum_{i = 1}^{m}|a_i|$\julienRevision{, where $|a_i|$ is the
% number of non-diagonal points in the diagram $a_i$; i.e. the pairwise
% augmentation procedure from \autoref{sec_metricSpace} is generalized to $m$
% diagrams by projecting on the diagonal of $\Bary$ the non-diagonal points of the $m$ diagrams of
% the dictionary $\Dict$.
% Then, in the following, we
% consider that $K = \sum_{i = 1}^{m}|a_i|$.
% }

% \julien{In practice, we compute

\julienRevision{Next, we optimize
$\Bary$ in practice with the approach by
Vidal et al. \cite{vidal_vis19}\julienRevision{, which provides a time-efficient
approximation of the original algorithm by Turner et al. \cite{Turner2014}.
Similar to Turner et al., it is based on an
iterative optimization, where each iteration includes
an \emph{Assignment} step, followed by an \emph{Update} step.
Specifically, the \emph{Assignment} step computes the optimal
assignments $\psi_i$ between the candidate $\Bary$ and each dictionary diagram
$a_i$.
% of the dictionary $\Dict$.
Next, the \emph{Update} step minimizes the
Fr\'echet energy $\frechetEnergy$ under the current assignments $\psi_i$.
Since the $L^2$-Wasserstein distance considers the Euclidean distance as a cost
function
(\autoref{sec_metricSpace}), this minimization is achieved
by simply
placing each point of $\Bary$ at the arithmetic mean in the birth/death space of
its assigned points in the dictionary diagrams. This can be done since the
arithmetic mean minimizes the Fr\'echet energy defined respectively to Euclidean
distances (more sophisticated \emph{Update} procedures, e.g. based on an
optimization routine, would need to be
derived for other distances in the birth/death space).
% Note that this update procedure
After this \emph{Update} step, the subsequent
\emph{Assignment}
% step
further improves the assignments $\psi_i$, hence
decreasing the Fr\'echet energy constructively at each iteration.}}

\julienRevision{The approach
by Vidal et al. \cite{vidal_vis19} revisits this framework by integrating
tailored approximations throughout
% an approximation procedure
% at the core of
the computation. Specifically, it approximates the optimal
assignments $\psi_i$ with the fast \emph{Auction} optimization
\cite{Bertsekas81} (instead of the traditional, yet prohibitive, \emph{Munkres}
algorithm
\cite{Munkres1957}). Further, it improves performance with a mechanism called
\emph{price memorization}, which enables the initialization of the \emph{Auction}
optimization with the assignments $\psi_i$ computed in the previous \emph{Assignment}
step. This allows the barycenter optimization to resume the assignment
optimization instead of re-computing it from scratch
% preventing a
% from-scratch re-computation
% of the assignments
at each iteration.
% of the barycenter optimization.
% Then, the assignments  $\psi_i$ are not re-computed from scratch at each
% \emph{Assignment} step and significant accelerations are observed in practice
% when the optimal assignments only evolve mildly from one \emph{Assignment}
% step to the next.
This approach also includes a strategy for the adaptive increase of the
\emph{accuracy} parameter of the \emph{Auction} optimization, allowing for fast
% and coarse
assignments in the early iterations of the barycenter algorithm, and
slower but more accurate assignments towards its convergence.}

\section{Wasserstein Dictionary Encoding} \label{sec:dictEnc}

\julien{This section formalizes our approach for the Wasserstein dictionary
encoding of an ensemble of persistence diagrams.
\autoref{sec_overview} provides an overview of our approach, which
% \autoref{section:Grad} describes our gradient descent strategy,
% interleaving
interleaves
barycentric weight optimization ($\Lamb$) with atom optimization ($\Dict$).
Finally, Secs. \ref{subsect:weight} and \ref{subsect:atom} detail the gradient
estimation for both sub-problems.
}

% In this section we formalize our approach of Wasserstein dictionary encoding on
% persistence diagrams. We present the optimization of the barycentric weights
% $\Lamb$ and the dictionary $\Dict$.

\subsection{Overview}
\label{sec_overview}

\julien{Let $\{X_1, \hdots , X_N\}$ be the input ensemble of $N$
persistence diagrams. The goal of our approach is to jointly optimize two
sub-problems:
\begin{itemize}
 \item Optimize a set $\Dict$ of $m$ \emph{reference}
persistence diagrams, called
% in the following
the \emph{atoms} of the \emph{Wasserstein dictionary}
$\Dict$;
 \item Optimize for each input diagram $X_n$ a vector of $m$ barycentric weights
$\Lamb_n \in \Sigma_m$,
% [0, 1]^m$,
in order to accurately approximate $X_n$ with a
Wasserstein barycenter $\BaryAll$
% , relative to $\Lamb$ and $\Dict$
(\autoref{sec:bary}).
\end{itemize}
}

\julien{This can be formalized as a joint optimization,
% expressed as an optimization problem,
where one wishes to
find the optimal barycentric weights $\Lambda^* = \Lamb_1^* , \hdots, \Lamb_N^*$
and the optimal Wasserstein dictionary $\Dict_{*} = \{a_1^*, \hdots, a_m^*\}$
(with $m \ll N$),
in
order to minimize the following \emph{dictionary energy}:
\begin{eqnarray}
  \label{eq_dictionaryEnergy}
 \dictionaryEnergy(\Lambda, \Dict) = \sum\limits_{n = 1}^{N} \War\big(\BaryAll
,X_n\big).
\end{eqnarray}
}

\julien{Our overall strategy for optimizing \autoref{eq_dictionaryEnergy}
consists
in
iteratively
interleaving two sub-optimizations:
\begin{enumerate}
 \item For a fixed dictionary $\Dict$,
 the set of
barycentric weights $\Lambda$ is optimized with one step of gradient descent
(\autoref{subsect:weight});
  \item For a fixed set of barycentric weights $\Lambda$,
  the dictionary $\Dict$ is optimized with one step of
%   we optimize by
gradient descent
% the dictionary $\Dict$
(\autoref{subsect:atom}).
\end{enumerate}
Then, this sequence of two sub-procedures is
iterated until a pre-defined stopping
condition is reached (\autoref{section:Prog}).
}

\julien{Finally, the output of our approach is
the optimized Wasserstein dictionary $\Dict_{*}$ (a set of $m$ atom
diagrams) and, for each input diagram $X_n$, a vector of
% barycentric
weights $\Lamb_n^* \in \Sigma_m$, which can be
interpreted
% viewed
as the
barycentric coordinates
of $X_n$ in \julienRevision{$\Dict_{*}$} (thus capturing linear relations between the input
diagrams on the Wasserstein dictionary).}

\subsection{Weight optimization} \label{subsect:weight}

\julien{This section details the optimization of the barycentric weights
$\Lambda = \Lamb_1 , \hdots, \Lamb_N$. Let $\Dict = \{a_1 , \ldots , a_m \}$ be
a
fixed dictionary of atom diagrams, with $m > 0$.
% \toAddress{We suppose in the
% following that all the diagrams have been augmented, as detailed
% in \autoref{sec_metricSpace}. Thus, at this stage, all the atoms $a_i$ (with
% $i \in \{1, \ldots , m\}$) have the same size $K > 0$: $a_i = \{ a_i^1, \hdots ,
% a_i^K \}$.}
}
\julien{Let $X$ be a diagram of the input ensemble. For a given set of weights
$\Lamb = \big(\lambda_1 , \hdots , \lambda_m \big)$, let $\BaryLamb =
\big\{y^1(\Lamb),\hdots, y^K(\Lamb)\big\}$ be its barycentric approximation,
relative to $\Dict$ \julienRevision{(i.e. each point $y^j(\Lamb)$ of
$\BaryLamb$ approximates a point in $X$)}.}

\julienRevision{We recall that after augmentation (\autoref{sec:bary}),
$\BaryLamb$ and the atoms contain $\sum_{i = 1}^{m}|a_i| + |\BaryLamb|$ points each,
where $|a_i|$ and $|\BaryLamb|$ denote the number of \emph{non-diagonal} points in
$a_i$ and $\BaryLamb$
respectively.
Then, in order to compare it to $X$,
$\BaryLamb$ is further augmented by projecting on its diagonal the $|X|$
non-diagonal points of $X$. Then, at this stage, the size $K$ of $\BaryLamb$ is
given by
$K = \sum_{i = 1}^{m}|a_i| + |\BaryLamb| + |X|$. We augment similarly $X$ (i.e. by
projecting the non-diagonal points of $\BaryLamb$ to its diagonal) and the $m$ atoms (i.e. by
projecting the non-diagonal points of $X$ to their diagonals). Then, at this
point,
$\BaryLamb$, $X$, and
% all
the $m$ atoms $a_i$
% (with
% $i \in \{1, \ldots , m\}$)
all have the same size $K = \sum_{i = 1}^{m}|a_i| + |\BaryLamb| + |X|$.}
% : $a_i = \{ a_i^1, \hdots ,
% a_i^K \}$.}

\julien{In this section, we describe a
% step of
gradient descent on $\Lamb$
% approach
to minimize the \emph{weight energy}:
\begin{eqnarray}
% \nonumber
\label{eq_weightEnergy}
\weightEnergy(\Lamb) = \War\big(\BaryLamb ,X\big).
\end{eqnarray}
A step of the corresponding gradient descent is illustrated in
\autoref{fig:weight}.}

\begin{figure}
 \centering % avoid the use of \begin{center}...\end{center} and use \centering instead (more compact)
 \includegraphics[width=\columnwidth]{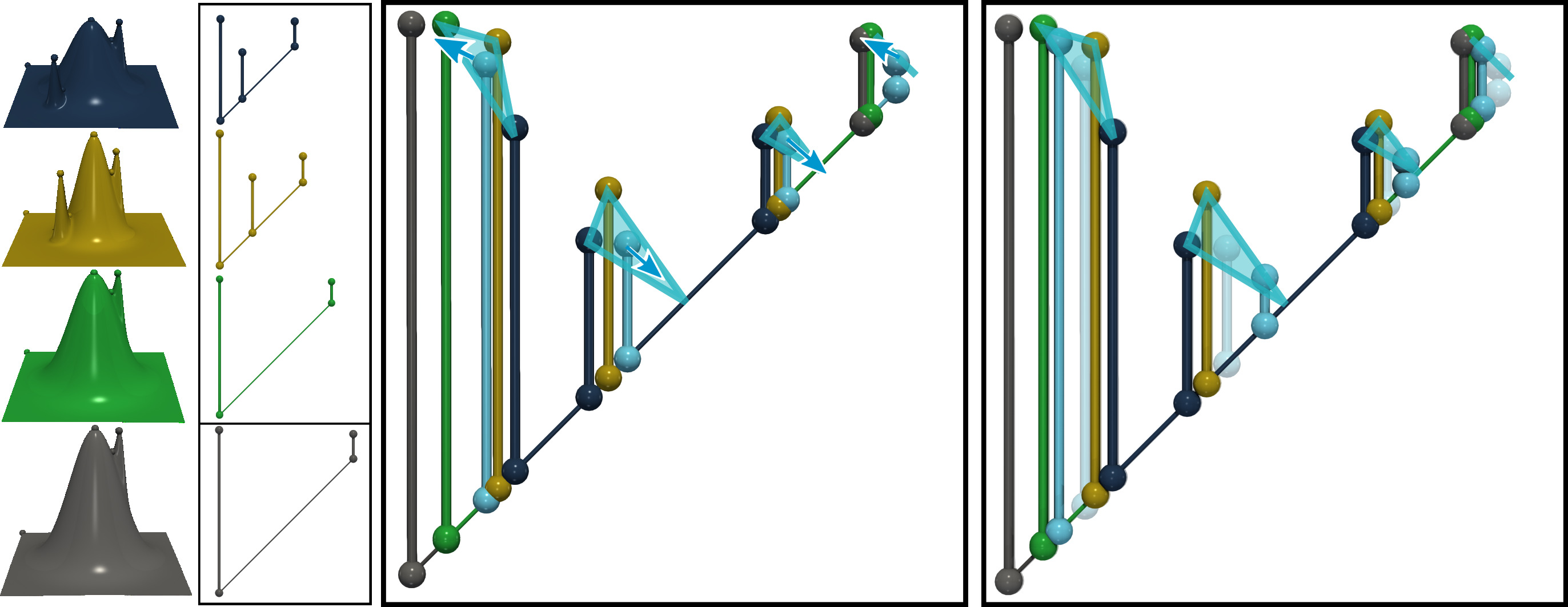}
 \caption{\julienRevision{Optimizing the weights of the
%  Wasserstein
 barycenter
$\BaryLamb$ (cyan diagram) to improve its approximation of $X$ (grey diagram),
given a fixed Wasserstein dictionary $\Dict$ of $3$ atoms (dark blue, yellow,
green). At a given iteration $t$ (center), a step $\rho_{\Lamb}$ is made along
the
gradient of the weight energy $\weightEnergy$ (cyan arrows), resulting in an
improved estimation at iteration $t+1$ (right).}}
%  \caption{Example of weights optimization, in this case we want to approximate
%  the grey persistence diagram by the barycenter (cyan) of three persistence
% diagrams (dark blue, dark yellow, and green). From left to right we have the
% scalar fields, the persistence diagrams associated (with the same color code),
% then the initial state before optimization, and finally the state after a few
% optimization steps. In the left square, the arrows represent the direction each
% pair in the barycenter has to take and is obtained by the matchings between the
% barycenter and the objective diagram; those directions are stored in the
% gradient, and after 2 gradient steps on the weights we obtain new barycentric
% weights, giving us a new barycenter in the right square.}
 \label{fig:weight}
\end{figure}

\julien{Given the set of optimal matchings $\phi_{1} , \hdots ,
\phi_{m}$ between $\BaryLamb$ and the $m$ atoms,
% each point
the $j^{th}$ point of $\BaryLamb$, noted $y^j(\Lamb)$,
% $y^j(\Lamb) \in \BaryLamb$
is given by:
% can be expressed as:
\begin{eqnarray}
% \nonumber
\label{eq_diagramCombination}
 \forall j \in \{1,\hdots, K\},\ y^j(\Lamb) = \sum\limits_{i=1}^{m}\lambda_i
a_i^{\phi_{i}(j)}.
\end{eqnarray}
}

\julienRevision{In other words, the $j^{th}$ point $y^j(\Lamb)$ of the diagram
$\BaryLamb$ is a linear combination (with the weights $\Lamb$) of the $m$ points
it matches to in the atoms (one point per atom $a_i$),
% in each atom $a_i$,
under the optimal assignments $\phi_{i}$
(i.e. minimizing \autoref{wassDist}).}

\julien{For a fixed set of assignments  $\phi_{1} , \hdots ,
\phi_{m}$, the Wasserstein distance (\autoref{wassDist}) between $X$ and
its approximation
$\BaryLamb$ is then:
% can be written as:
\begin{eqnarray}
  \nonumber
 \weightEnergy(\Lamb) = \War\big(\BaryLamb ,X\big) =
 \julienRevision{\sum\limits_{j=1}^{K}
 c\bigl(y^j(\Lamb), x^{\psi(j)}\bigr),}
% \biggl(\sum\limits_{i=1}^{m} \lambda_i a_i^{\AssI(j)}\biggr) - } \|^2,
\end{eqnarray}
where \julienRevision{$\psi$} denotes the optimal assignment \julienRevision{(\autoref{wassDist})} between $X$ and
its approximation
$\BaryLamb$.}
\julienRevision{When $y^j(\Lamb)$ and $x^{\psi(j)}$ are not both diagonal
points, the cost
$c\bigl(y^j(\Lamb), x^{\psi(j)}\bigr)$ is given by their squared Euclidean
distance in
the birth/death space (it is zero otherwise, see \autoref{sec_metricSpace}).
Then, by exploiting \autoref{eq_diagramCombination}, $\weightEnergy(\Lamb)$ can be re-written as:
\begin{eqnarray}
  \nonumber
 \weightEnergy(\Lamb)  =  \War\big(\BaryLamb ,X\big)  & = &
 \sum\limits_{j=1}^{K} \|
  y^j(\Lamb) - x^{\psi(j)} \|^2 \\
  \nonumber
 & = &
 \sum\limits_{j=1}^{K} \|
\biggl(\sum\limits_{i=1}^{m} \lambda_i a_i^{\AssI(j)}\biggr) - x^{\psi(j)} \|^2.
% \biggl(\sum\limits_{i=1}^{m} \lambda_i a_i^{\AssI(j)}\biggr) - } \|^2,
\end{eqnarray}
Since $\sum\limits_{i=1}^{m} \lambda_i = 1$, $\weightEnergy(\Lamb)$ can
finally be re-written as:}
% \begin{eqnarray}
%   \nonumber
%  \weightEnergy(\Lamb) = \War\big(\BaryLamb ,X\big) =
%  \sum\limits_{j=1}^{K} \|
% \biggl(\sum\limits_{i=1}^{m} \lambda_i a_i^{\AssI(j)}\biggr) - x^{\psi(j)} \|^2,
% \end{eqnarray}

\begin{eqnarray}
  \label{eq_energyFixedAssignment}
%   \nonumber
 \weightEnergy(\Lamb) = \War\big(\BaryLamb ,X\big) =  \sum\limits_{j=1}^{K} \|
\sum\limits_{i=1}^{m} \lambda_i (a_i^{\AssI(j)} - x^{\psi(j)}) \|^2.
\end{eqnarray}

\julienRevision{Intuitively, this energy measures the error (in terms of
Wasserstein distance) induced by
approximating the input diagram $X$ with its barycentric approximation
$\BaryLamb$.
 In \autoref{eq_energyFixedAssignment},
 it is computed for each $j^{th}$ point $y^j(\Lamb)$
 of the
 diagram
 $\BaryLamb$, by considering the birth/death distances between the
 points  $y^j(\Lamb)$ maps to, in the atoms on one hand and in the input diagram
 $X$
 on the other.}
%  by summing the birth/death distance between each point of
% for each point $y^j(\Lamb)$ of the diagram $\BaryLamb$, the (weighted) cost
% induced by
% assigning its matching point
% $a_i^{\AssI(j)}$ in the atom $a_i$ and its matching point $x^{\psi(j)}$ in $X$.
% In other words, this energy can be viewed as a weighted combination of the
% distances between $X$ and the atoms, which are evaluated on a per point

% Thus,
\julienRevision{Then,}
by applying the chain rule on
\autoref{eq_energyFixedAssignment}, the gradient of the weight energy
(\autoref{eq_weightEnergy}) is given by:
\begin{equation}
\label{eq_weightGradient}
 \nabla \weightEnergy(\Lamb) =
%  \War\big(\BaryLambO , X\big) =
 2 \sum\limits_{i=1}^{m} \sum\limits_{j=1}^{K}
\begin{bmatrix}
(a_1^{\Ass(j)} - x^{\psi(j)})^T \\
\vdots \\
(a_m^{\AssM(j)} - x^{\psi(j)})^T
\end{bmatrix}\big(\lambda_i (a_i^{\AssI(j)} -
x^{\psi(j)})\big).
\end{equation}

Now that the gradient of the weight energy is available
(\autoref{eq_weightGradient}), we can proceed to gradient descent. Specifically,
the barycentric weights at the iteration $t+1$ (noted $\Lamb^{t+1}$)
are obtained by a step $\rho_{\Lamb}$ from the weights at the iteration $t$
(noted
$\Lamb^{t}$) along the gradient:
\begin{equation}
  \label{eq_lambdaUpdate}
%  \nonumber
 \Lamb^{t+1} = \Pi_{\Sigma_m} \big(\Lamb^t - \rho_{\Lamb} \nabla
 \weightEnergy(\Lamb^{t}) \big),
% \War\big(\BaryLambO , X\big) \right)
\end{equation}
where $\Pi_{\Sigma_m}$ is the projection onto the simplex of admissible
barycentric weights (i.e. positive and summing to $1$, c.f. \autoref{sec:bary}).
Since \julienQuestion{$\keanu{\nabla} \weightEnergy$} is $L$-\julien{Lipschitz}
(see the \julien{computation details}
% corresponding proof
in
\julienQuestion{Appendix A}), \julien{a gradient step will guarantee an energy
decrease
as long as:}
% we choose $\rho_{\Lamb}$ such that:
\begin{equation}
  \label{eq_gradientStep}
  \nonumber
  \rho_{\Lamb} \leq  \left[2 \displaystyle\sum\limits_{j=1}^{K}  \begin{Vmatrix}
 (a_1^{\Ass(j)} - x^{\psi(j)})^T \\
\vdots \\
 (a_m^{\AssM(j)} - x^{\psi(j)})^T
\end{Vmatrix}^{2}\right]^{-1} < \frac{1}{L}.
\end{equation}

Overall, for a given input diagram $X$, each iteration $t$ of gradient descent
for the optimization of
$\weightEnergy$ consists in the following steps:
\begin{enumerate}[leftmargin=0.75cm]
 \item Computing the Wasserstein barycenter $\BaryLambO$
(\autoref{sec:bary});
 \item Computing the Wasserstein distance $\War\big(\BaryLambO ,X\big)$
(\autoref{eq_weightEnergy});
  \item Estimating the gradient $\nabla \weightEnergy(\Lamb)$
(\autoref{eq_weightGradient});
  \item{Applying one step $\rho_{\Lamb}$ of gradient descent
(\autoref{eq_lambdaUpdate}).}
\end{enumerate}

\subsection{Atom optimization} \label{subsect:atom}

\begin{figure}
 \centering % avoid the use of \begin{center}...\end{center} and use \centering instead (more compact)
 \includegraphics[width=\columnwidth]{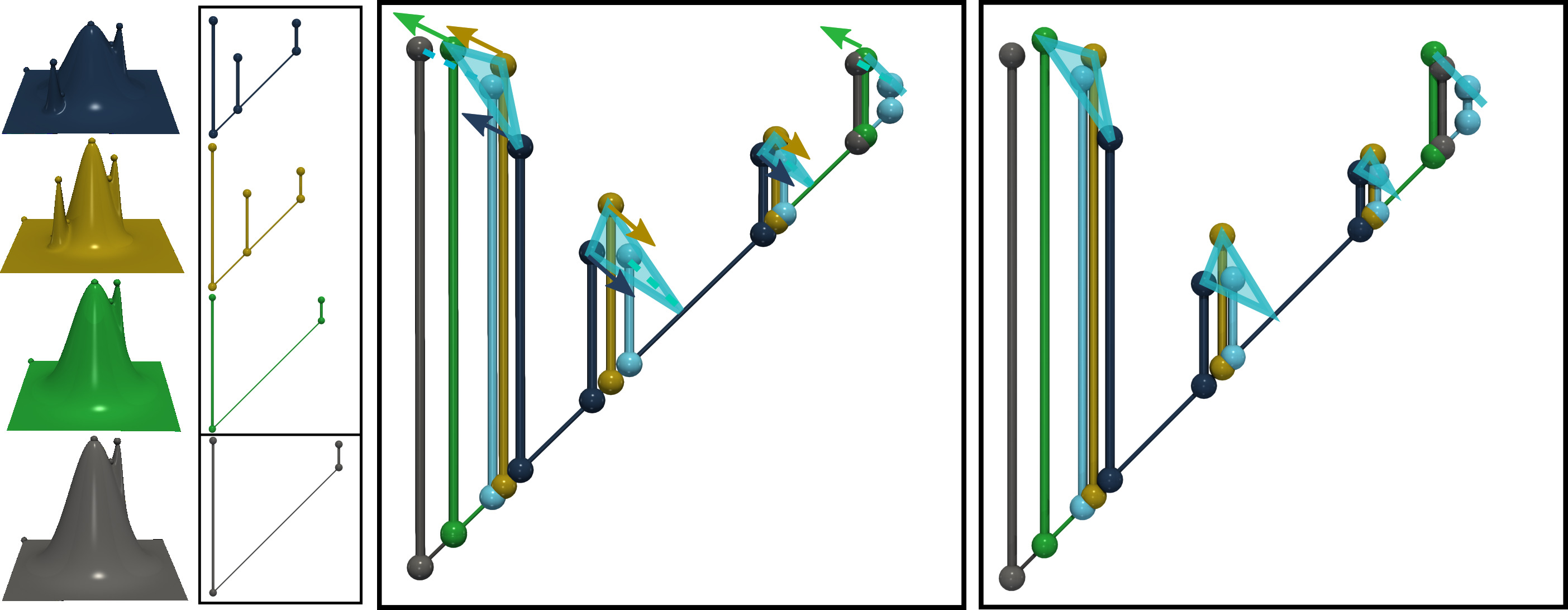}
 \caption{\julienRevision{Optimizing the atoms of the Wasserstein dictionary $\Dict$
(dark blue, yellow and green diagrams). At a given iteration $t$ (center), a
step $\rho_{\Dict}$ is made along the gradient of the pointwise atom energy
$\individualAtomEnergy$ (arrows on each triangle), resulting in a dictionary
(right) that enables an improved barycentric approximation ($\BaryDict$, cyan)
of the input diagram $X$ (grey).}}
%  \caption{Example of atom optimization, the setting is the same as in
% \autoref{fig:weight}. In the left square, the arrows represent the direction
% each pair in the atoms has to take and is obtained using the matchings between
% the objective diagram and the barycenter along those between the barycenter and
% the atoms; those directions are stored in the gradient, and after a few gradient
% steps on the atoms we obtain a new barycenter in the right square.}
 \label{fig:atom}
\end{figure}

\julien{This section details the optimization of the atoms of the dictionary
$\Dict = \{a_1 , \ldots , a_m\}$.
% (typically obtained with the optimization procedure
% described in \autoref{subsect:weight}).
Similarly to
\autoref{subsect:weight},
% the previous sub-section,
let $X$ be a diagram of the input ensemble and
let $\Lamb \in \keanu{\Sigma_m}$ be its -- fixed -- vector of
barycentric weights.
For a given dictionary $\Dict$, let
$\BaryDict = \big\{ y^1(\Dict), \hdots , y^K(\Dict) \big\}$ be the barycentric
approximation of $X$, relative to $\Lamb$. In this section, we describe a step
of
gradient descent on $\Dict$ to minimize the following \emph{atom energy}:
\begin{equation}
  \label{eq_atomEnergy}
  \nonumber
  \atomEnergy(\Dict) = \War\big(\BaryDict ,X\big).
\end{equation}
A step of the corresponding gradient descent is illustrated in
\autoref{fig:atom}.
}

\julien{Given the set of optimal matchings $\phi_{1} , \hdots ,
\phi_{m}$ between $\BaryDict$ and the $m$ atoms,
% each point
the $j^{th}$ point of $\BaryDict$, noted $y^j(\Dict)$,
% $y^j(\Lamb) \in \BaryLamb$
is
given by:
% can be expressed as:
\begin{eqnarray}
\nonumber
 \forall j \in \{1,\hdots, K\},\ y^j(\Dict) = \sum\limits_{i=1}^{m}\lambda_i
a_i^{\phi_{i}(j)}.
\end{eqnarray}
\julienRevision{This expression is identical to \autoref{eq_diagramCombination}
(\autoref{subsect:weight}).
However, $y^j$ now depends on $\Dict$, which is the variable of the current optimization.
Then,}
%
% It follows that
the gradient of $y^j(\Dict)$ with regard to $\Dict$
is \julienRevision{simply} given by:
% can be
% written as:
\begin{eqnarray}
  \label{eq_pointGradient}
%   \nonumber
  \nabla y^j(\Dict) = \begin{bmatrix}
\lambda_1 & \cdots &\lambda_m
\end{bmatrix}^T.
\end{eqnarray}
}

For a fixed set of assignments  $\phi_{1} , \hdots ,
\phi_{m}$, the Wasserstein distance (\autoref{wassDist}) between $X$ and
its approximation
$\BaryLamb$ is then:
% can be written as:
\begin{eqnarray}
  \label{eq_energyDictFixedAssignment}
  \nonumber
 \atomEnergy(\Dict) = \War\big(\BaryDict ,X\big) \julienRevision{=  \sum\limits_{j=1}^{K}
 c\bigl(y^j(\Dict), x^{\psi(j)}\bigr),}
\end{eqnarray}
where $\psi(j)$ denotes the optimal assignment between $X$ and
its barycentric approximation
$\BaryDict$.
\julienRevision{Similarly to \autoref{eq_energyFixedAssignment}
(\autoref{subsect:weight}), the above equation can
be re-written as:
\begin{eqnarray}
  \label{eq_energyDictFixedAssignment}
  \nonumber
 \atomEnergy(\Dict) = \War\big(\BaryDict ,X\big) =  \sum\limits_{j=1}^{K} \|
\sum\limits_{i=1}^{m} \lambda_i (a_i^{\AssI(j)} - x^{\psi(j)}) \|^2.
\end{eqnarray}}

Let $\Dict^j = [a_{1}^{\phi_{1}(j)}, \hdots ,
a_{m}^{\phi_{m}(j)}]^T$ be the $(m \times 2)$-matrix formed by the atom points
matching to a given point $y^j(\Dict)$ of $\BaryDict$, via the fixed
assignments $\phi_{1} , \hdots ,
\phi_{m}$.
% (mx2) matrix: m lines, 2 columns
\julienRevision{Specifically, the $i^{th}$ line of this matrix refers to the point
$a_{i}^{\phi_{i}(j)}$ in the
atom $a_i$ where $y^j(\Dict)$ maps to (via the optimal assignment $\phi_i$).
For this line, the two
columns of the matrix encode the birth/death coordinates of the point
$a_{i}^{\phi_{i}(j)}$.}
% For a given point $y^j(\Dict)$ of $\BaryDict$, t
Then, the
\emph{pointwise atom energy} of $y^j(\Dict)$, noted
$\individualAtomEnergy(\Dict^{j})$, is  given by:
\begin{eqnarray}
%   \nonumber
  \label{eq_individualAtomEnergy}
  \individualAtomEnergy(\Dict^{j}) =
 \left\lVert \sum\limits_{i=1}^{m} \lambda_i
(a_{i}^{\AssI(j)} - x^{\psi(j)}) \right\rVert^2.
\end{eqnarray}

\julien{Then, by applying the chain rule on \autoref{eq_individualAtomEnergy}
(using \autoref{eq_pointGradient}), the gradient of the pointwise atom energy
is given by:
\begin{eqnarray}
\label{eq_pointWiseGradient}
 \nabla \individualAtomEnergy(\Dict^{j}) = 2
%  \begin{bmatrix}
% \lambda_1\\
% \vdots\\
% \lambda_m
% \end{bmatrix}
\begin{bmatrix}
\lambda_1 & \cdots &\lambda_m
\end{bmatrix}^T
\big(\sum\limits_{i=1}^{m} \lambda_i (a_i^{\AssI(j)} -
x^{\psi(j)})^T\big).
\end{eqnarray}
}

\julien{Now that the gradient of the pointwise atom energy is available
(\autoref{eq_pointWiseGradient}), we can proceed to a step of gradient descent.
Specifically, the matrix of atom points matched to $y^j(\Dict)$ at the iteration
$t+1$ (noted $\Dict^j_{t+1}$) is obtained by a step $\rho_\Dict$ from the same
matrix at the iteration $t$ (noted $\Dict^j_{t}$) along the gradient:
\begin{equation}
% \nonumber
\label{eq_atomUpdate}
\Dict_{t+1}^j = \Pi_\mathcal{X} \big( \Dict_t^j - \rho_{\Dict} \nabla
\individualAtomEnergy(\Dict_t^{j}) \big),
\end{equation}
where $\Pi_\mathcal{X}$ projects
% is a projection which guarantees that
each atom point
% gets
% relocated
to an admissible region of the 2D birth/death space (i.e. above the
diagonal and within the global scalar field range). \keanu{Since
\julienQuestion{$\keanu{\nabla}\individualAtomEnergy$} is $L$-\julien{Lipschitz}
% with $L <
% 4m$
(see the
% corresponding proof
\julien{computation details}
in
\julienQuestion{Appendix B}),
\julien{a gradient step will guarantee an energy decrease as long as:}
% we choose $\rho_\Dict$ such that
$\rho_\Dict < (4m)^{-1} < L^{-1}$.}
}

\julienRevision{Note that, in order to control the final size $S_m$ of the
dictionary $\Dict$, after each iteration of atom optimization, each atom $a_i$
is thresholded by removing
% its remaining diagonal points (if any) as well as
its
$\overline{K} = (m K - S_m) / m$ least persistent points (at the subsequent optimization
iteration, all diagrams will be re-augmented again in a pre-preprocess, as
detailed in \autoref{subsect:weight}).}
% The,
% Then, the augmentation procedure of
% $\BaryLamb$ is run again (\autoref{subsect:weight})
%
% the $(K - S_m) / m$ least persistent points of the atom $a_i$
% each atom $a_i$ is
% thresholded and only the }

\julien{Overall,  for a given input diagram $X$,  each iteration $t$ of gradient
descent for the optimization of $\atomEnergy$ consists in the following steps:
\begin{enumerate}[leftmargin=0.75cm]
 \item Computing the Wasserstein barycenter $\BaryDictO$
(\autoref{sec:bary});
 \item Computing the Wasserstein distance $\War\big(\BaryDictO ,X\big)$
(\autoref{eq_weightEnergy});
  \item For each point $y^j(\Dict)$ of $\BaryDict$:
    \begin{enumerate}
      \item Estimating the gradient $\nabla \individualAtomEnergy(\Dict^{j})$
(\autoref{eq_pointWiseGradient});
      \item Applying one step $\rho_\Dict$ of gradient descent on $\Dict^{j}$
(\autoref{eq_atomUpdate});
    \end{enumerate}
  \item \julienRevision{Remove the
%   $(m K - S_m) / m$
  $\overline{K}$
  least persistent points from each atom $a_i$.}
%   $a_i$.}
%   and re-augment $\BaryLamb$ and the atoms.}
\end{enumerate}
}

\section{Algorithm}
\label{sec_algorithm}

\julien{This section presents our overall algorithm for the resolution of the
optimization
formulated in \autoref{sec:dictEnc}.
\autoref{section:init} details our initialization strategy.
Our overall
% progressive
multi-scale
scheme is presented in \autoref{section:Prog}. Finally,
shared-memory parallelism is discussed in \autoref{section:Para}.}

% In this section we present our overall framework based on the formalism introduced in \autoref{sec:dictEnc}. We introduce empirical strategy helping the basic optimization algorithm for the barycentric weights and dictionary of persistence diagrams.
%
% \subsection{Overview}
%
% There are three key insights of our algorithm. First the \textit{initialization} of our dictionary is inspired from the k-means++ \cite{celebi13} initialization and is detailed in \autoref{section:init}.
% Second the optimization of the barycentric weights and the atoms is an alternate gradient descent and the process is detailed in \autoref{section:Grad}. Finally, progressivity can be injected by controlling the resolution of the input diagrams (see \autoref{section:Prog}) . Our framework is also parallelizable (\autoref{section:Para}), the progressivity and controlling the size of our atoms allow our approach to have respectable running time.

\begin{figure}[tb]
 \centering % avoid the use of \begin{center}...\end{center} and use \centering
% instead (more compact)
 \includegraphics[width=\columnwidth]{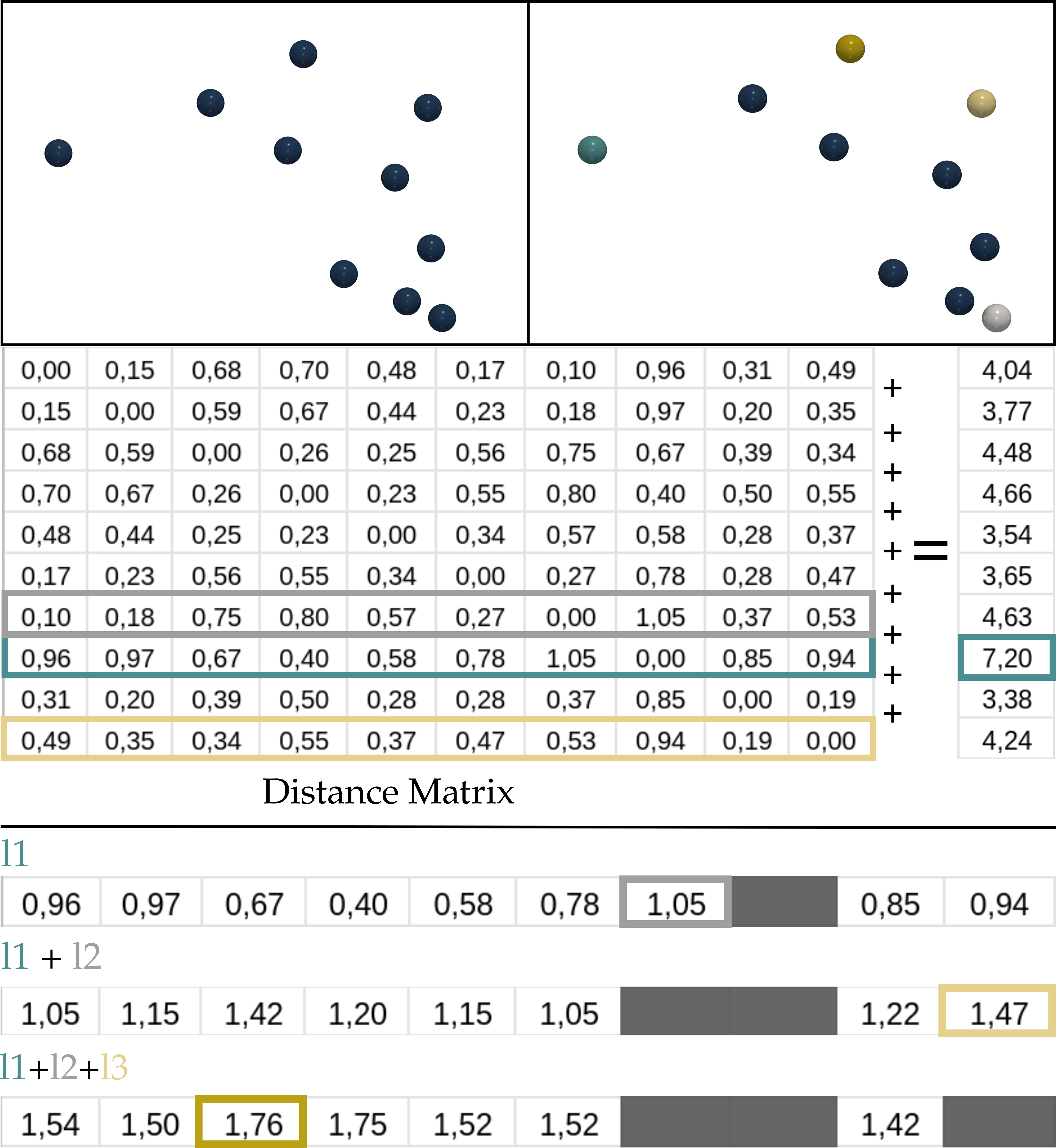}
 \caption{\julien{Illustration of our initialization strategy on a toy 2D
point set (top left). First, the entries of the distance matrix of the input
(middle) are summed on a per-line basis. The line maximizing this sum (cyan),
noted $l_1$, identifies the
first atom, noted $a_1$, as the point
% point $p_1$
which is the \emph{furthest away} from
all the others (cyan sphere, top right).
Next, the atom $a_2$ (grey sphere, top right) is selected as the point which
maximizes its distance to $a_1$. At this point, the line $l_2$
(grey, corresponding to the point $a_2$) is added to the line $l_1$, to encode
the distances to these two
atoms
% reference points
($a_1$ and $a_2$). Then, the point
$a_3$ is selected as the maximizer of $l_1 + l_2$: it is the point which is the
furthest away from all the previously selected
atoms.
% reference points.
Then, the
corresponding line, $l_3$, is added to $l_1 + l_2$ and the process is iterated
until the target number of
% reference points
atoms
has been achieved.}}
%  \caption{Simple example of the initialization method on a sample of 2D
% points.
% Firstly a distance matrix from the points is computed, then all the columns
% are
% summed and the maximizer (here the eight-th line) of the column obtained is
% our
% first element corresponding to the farthest point from all the others in the
% sample. The second point is chosen by looking at the eight-th line and taking
% the maximizer of it (here the seventh column), corresponding to the farthest
% point from the first element chosen. The next point is taken by taking the
% maximizer (here the tenth column) of the line obtained by summing the eight-th
% and seventh line together, representing the farthest point from the two of
% them.
% And finally, we sum the all the previous lines (eight-th, seventh and tenth)
% and
% take its maximizer (third column here) to get the final point.}
 \label{fig:init}
\end{figure}

\subsection{Initialization} \label{section:init}

\julien{Our strategy for the initialization of the Wasserstein dictionary
$\Dict$, illustrated in \autoref{fig:init}, is inspired by the celebrated
\emph{k-means++} strategy \cite{celebi13}.
Specifically, we iteratively select the $m$ atoms among the $N$ input diagrams.
At the first iteration, we select as first atom the diagram which maximizes the
sum of its Wasserstein distances (\autoref{wassDist}) to all the input diagrams
(cyan point in \autoref{fig:init}). Next, each iteration selects as the next
atom the diagram which maximizes the sum of its Wasserstein distances to all
the previously selected atoms. This process stops when the desired number of
atoms, $m$, has been selected. As illustrated in \autoref{fig:init} in the case
of a toy 2D point set, this initialization strategy has the nice property that
it tends to select atoms on the convex hull of the input point set,
which ensures that the non-atom points can indeed be expressed as
a convex combination of the atoms, hence leading to accurate initial
barycentric approximations.} \julien{As for the barycentric weights,
these are uniformly initialized (i.e. to $1/m$).}
% This is useful for persistence diagrams,
% In the case of persistence }

% by selecting at each iteration the diagram which maximizes the sum of its
% Wasserstein distances (\autoref{wassDist}) to all the previously selected atom
% diagrams (for the first iteration, we select the diagram which maximimizes teh
% ).
% say it tens to select reference point on the convex hull of the point cloud,
% which is a desirable property for our optimization.

% Inspired by the k-means++ \cite{celebi13}, we initialize our dictionary of size
% $m$ by choosing the $m$ farthest diagrams in our initial dataset, as illustrated
% in \autoref{fig:init}. Usually in our paper, for clustering, we chose $m$ to be
% equal to the number of clusters in our original set of persistence diagrams. But
% for the application in dimensionality reduction we set the number of atoms $m$
% to be less than 3.

% \subsection{Persistence driven progressivity variant}
\subsection{Multi-scale optimization algorithm}
\label{section:Prog}

\julien{In real-life data, persistence diagrams tend to contain many
low-persistence features, which essentially encode the noise in the data (see
\autoref{fig:gaussianPDStability}, right). In this section, we present a
multi-scale optimization strategy which
% aims at addressing
addresses
this issue by
prioritizing the optimization on the most persistent pairs, which
correspond to the most salient features of the data. As detailed in
\autoref{section:frameQuality}, this strategy leads the optimization to
% local minima
solutions
of improved energy in comparison to a naive (non-multi-scale) approach.}

% \julien{
Our multi-scale strategy consists in iterating our optimization
procedure by progressively increasing the \emph{resolution} (in terms of
persistence) of the input
% persistence
diagrams. This is inspired by the
progressive strategy by Vidal et al. \cite{vidal_vis19} for the problem of
Wasserstein barycenter optimization. Specifically, given an input diagram $X$,
let $\Delta f$ be the span in scalar values in the \julienRevision{corresponding
% scalar field}
ensemble member}
$f$ (i.e. $\Delta f = \max_{v \in \mathcal{M}} f(v) - \min_{v \in\mathcal{M}}
f(v)$).
Given a threshold $\tau \in [0, 1]$, we note $X^{\tau} = \{ x \in X |\  d_x -
b_x \geq \tau \Delta f\}$ the version of $X$ at resolution $\tau$. It is a
subset of $X$ which contains persistence pairs whose relative persistence is
above $\tau$.
\julienRevision{Note that the input diagrams are not normalized by persistence,
which would prevent the capture of variability in data ranges within the ensemble.
Instead, we normalize the above persistence threshold, by expressing it as
a fraction $\tau \in [0, 1]$ of the scalar field range $\Delta f$.}

\begin{figure}[]
 \centering % avoid the use of \begin{center}...\end{center} and use \centering
% instead (more compact)
 \includegraphics[width=\columnwidth]{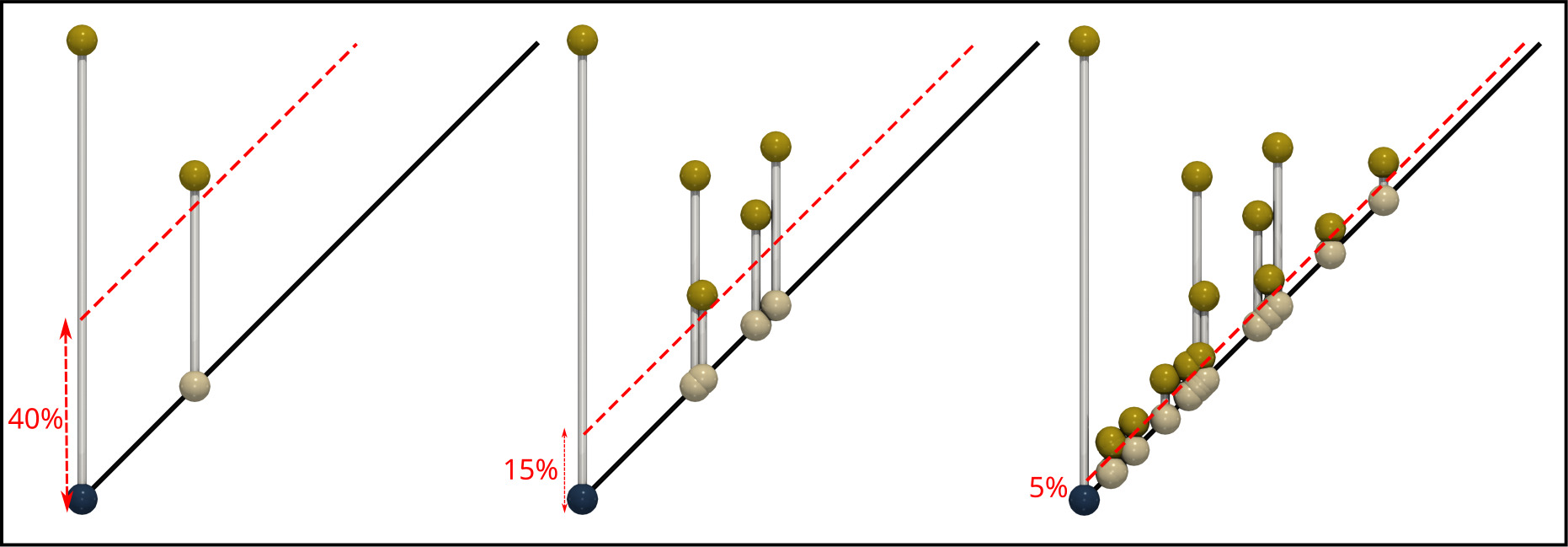}
 \caption{\julien{Multi-resolution representation of an input persistence
diagram (taken from the \emph{Isabel} ensemble). At a given resolution (from
left to right), only the persistence pairs above a given persistence threshold
(red dash line) are considered in the optimization.}}
%  In this figure we simply illustrate how the progressivity works,
%  the principle is filtering all the pairs below a certain threshold on the most
% persistent pair (the global pair), from left to right we have a diagram from
% the
% Isabel storm dataset thresholded to 40\% of the persistence of its global pair,
% then lowering the threshold to 15\% letting 2 more pairs, and finally lowering
% the threshold to 5\%.}
 \label{fig:prog}
\end{figure}

Then, our multi-scale optimization will first consider the input
diagrams at a resolution $\tau_0$ and then will
progressively consider finer resolutions $\tau_1, \dots, \tau_r$ until the full
diagrams are considered at $\tau_r = 0$.
\julienRevision{This multi-resolution strategy, based on a per-diagram normalized persistence threshold
($\tau \in [0, 1]$) prevents diagrams from being empty in the early resolutions
in case of large variations in data range within the ensemble (which would
occur for instance with a per-ensemble normalization).}
% This
\julienRevision{The multi-resolution}
is illustrated in \autoref{fig:prog}.
In our experiments, we set $\tau_0 =
0.2$ and decrease $\tau$ by $0.05$ at each resolution (i.e. $\tau_1 = 0.15,
\tau_2 = 0.10, \tau_3 = 0.05, \tau_4 = 0$). At each resolution, the solution
for the previous resolution is used as an initialization.
\julienRevision{Note that alternative strategies were considered for decreasing
$\tau$ (for instance by dividing it by $2$ at each resolution), but the best
experimental results were obtained for the above decrease strategy.}

\julien{\autoref{alg:PDEnc} summarizes our overall approach. For each
sub-optimization (i.e. weight and atom optimization), although each gradient
step is guaranteed to decrease the corresponding energy (see the end of Secs.
\autoref{subsect:weight} and \autoref{subsect:atom}), this is only true for
fixed assignments (between a diagram $X$ and its barycentric approximation as
well as between the barycentric approximation and the atoms). Since
the assignments can change along the iterations of the optimization,
% can change during the optimization,
the overall energy
$\dictionaryEnergy$ (\autoref{eq_dictionaryEnergy}) may increase between
consecutive iterations. Hence, pragmatic stopping conditions need to be
considered. In practice, if $\dictionaryEnergy$ has not decreased for more than
% \julienQuestion{XX}
\keanu{10} iterations, we return the solutions $\Lamb^*$ and
$\Dict_{*}$ reached by the optimization with the lowest energy
$\dictionaryEnergy$.}

\begin{algorithm}[t]
\SetAlgoLined
\SetAlCapFnt{\scriptsize}
\SetAlCapNameFnt{\scriptsize}
\scriptsize
\KwIn{ Set of persistence diagrams $\{X_1, \hdots , X_N\}$;}
\textbf{Output 1:} Wasserstein Dictionary $\Dict_*$\;
\textbf{Output 2:} Barycentric weights $\Lamb_1^*, \hdots , \Lamb_N^*$\;
\For{$\tau \in \{\tau_0, \hdots, \tau_{r}\}$}{
	\If{$\tau == \tau_0$}{
		Initialization (\autoref{section:init})\;
	}
	\While{$\dictionaryEnergy$ (\autoref{eq_dictionaryEnergy}) decreases}{
		\For{\julienRevision{$n \in \{1, \hdots, N\}$}}{
			Perform a gradient step $\rho_{\Lamb\julienRevision{_n}}$ along $\nabla \weightEnergy$ \julienRevision{relative to $X_n$}
			(\autoref{subsect:weight})\;
		}
		\For{\julienRevision{$n \in \{1, \hdots, N\}$}}{
			Perform a gradient step $\rho_{\Dict}$ along $\nabla \atomEnergy$ \julienRevision{relative to $X_n$}
			(\autoref{subsect:atom})\;
		}
% 	  \For{$n \in \{1,\hdots,N\} $}{
%
% 	  	Compute $\BaryGenAll$ and store $\Ass^n, \hdots , \AssM^n$\;
% 	  	Compute $\War(\BaryGenAll,X_{n}^{\tau})$ and store $\psi_n$\;
% 	  	Compute $\nabla_{\Lamb}(\War(\BaryGenAll,X_{n}^{\tau}))$ using $\Ass^n,
% \hdots , \AssM^n$ and $\psi_n$\;
% 	  	$\Lamb_n^* \leftarrow \Lamb_n^* - \rho_n
% \nabla_{\Lamb}(\War(\BaryGenAll,\
% X_{n}^{\tau})) $\;
% 	  	$\Lamb_n^* \leftarrow \Pi_{\Sigma}(\Lamb_n^*)$
% 	  }
% 	  \For{$n\in \{1,\hdots,N\}$}{
% 	  	Compute $\BaryGenAll$ and store $\Ass^n, \hdots , \AssM^n$\;
% 	  	Compute $\War(\BaryGenAll,X_{n}^{\tau})$ and store $\psi_n$\;
% 	  	\For{$j\in\{1,\hdots,K\}$}{
%
% 	  		Compute $\nabla_{\Dict^j} g_{j,n}(\Dict_{*}^j)$ using $\Ass^n, \hdots
% ,
% \AssM^n$ and $\psi_n$\;
%
% 	  	}
% 	  }
% 	  \For{$n\in \{1,\hdots,N\}$}{
% 	  	\For{$j\in\{1,\hdots,K\}$}{
% 	  		$\Dict_{*}^j \leftarrow \Dict_{*}^j - \rho \nabla_{\Dict^j}
% g_{j,n}(\Dict_{*}^j)$\;
% 	  	}
% 	  }
% 	  $t\leftarrow t+1$\;
	 }
}
 \caption{Multi-scale Wasserstein Dictionary Optimization.}
%  ProgressivePersistence Diagrams Dictionary Encoding.}
 \label{alg:PDEnc}
\end{algorithm}

\subsection{Parallelism} \label{section:Para}

\julien{Our approach can be trivially parallelized with shared-memory
parallelism. First, its most computationally demanding task,
% the
% computation of
the $N$ barycentric approximations of the input diagrams can be computed
% run
independently.
% , on a per input diagram basis.
Thus, for each barycentric
approximation, we use one parallel task per input diagram.
% we use the Wasserstein barycenter
% % computation
% algorithm by Vidal et al. \cite{vidal_vis19}, which already includes
% parallelism. In our setting, this algorithm parallelizes the computation of the
% Wasserstein distances between a barycentric approximation and the atoms, on a
% per atom basis (i.e. one parallel task per atom).
Next, the estimation of the
gradient of $\weightEnergy$ (\autoref{subsect:weight}) is done on a per input
diagram basis, independently. Thus, we use one parallel task per input diagram.
Regarding the estimation of the gradient of $\atomEnergy$
(\autoref{subsect:atom}), given a barycentric approximation $\BaryDict$ of an
input diagram $X$, each of its points $y^j(\Dict)$ defines independently a
pointwise version of the gradient of the atom energy (see the last paragraph
of \autoref{subsect:atom}). Thus, we use one parallel task per point
$y^j(\Dict)$ of a barycentric approximation $\BaryDict$ of an input diagram
$X$.}

% Our framework can be trivially parallelized as the most computationally
% demanding task is the computation of the $N$ barycenters of the atoms $d_i$ as
% their are independent for each input $X_n$. The optimization of the barycentric
% weights being also independent, we can also parallelized them. Thus, we
% parallelize our approach by running the computation of the $N$ barycenters and
% the optimization of the $N$ barycentric weights in $n_t$ independent threads.

\section{Applications}
\label{sec_applications}

\julien{This section illustrates the utility of our approach in concrete
visualization tasks: data reduction and dimensionality reduction.}

% \julienQuestion{todo: discuss the choice for $m$, discuss the choice for the
% insertion of new points in the atoms. say that the resulting clustering is
% identical, should we discuss the insertion of the points in the atom? }

\begin{figure*}[tb]
 \centering % avoid the use of \begin{center}...\end{center} and use \centering instead (more compact)
 \includegraphics[width=\linewidth]{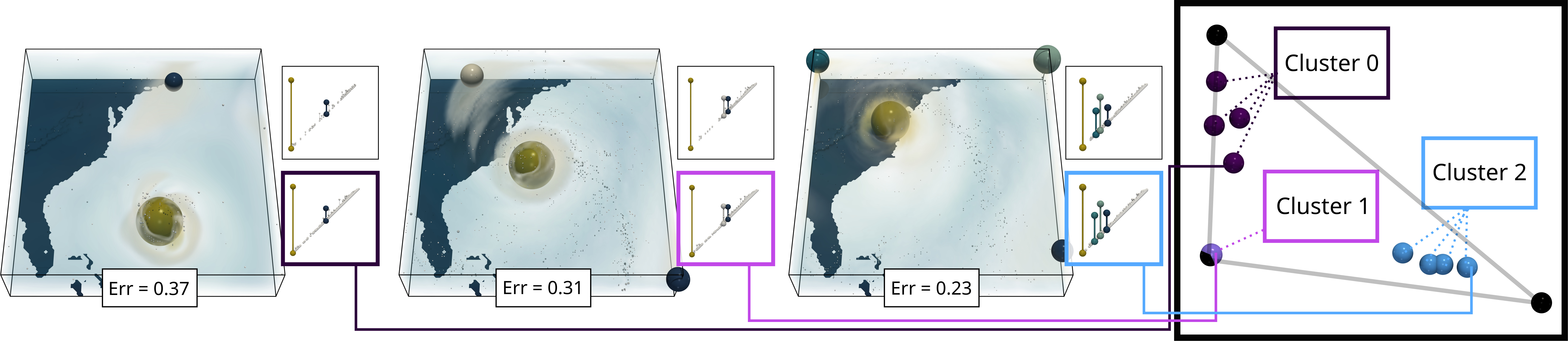}
 \caption{\julien{Visual comparison (left) between the input persistence
diagrams (top insets, saddle-maximum persistence pairs only) and our compressed
diagrams (bottom insets, \autoref{sec:
dataRed}, saddle-maximum persistence pairs only) for three members of the
\emph{Isabel} ensemble (one member per
ground-truth class). For each member, the sphere color encodes
the matching between
% the features of
the input and the compressed diagrams (for
the meaningful persistence pairs, above \julienQuestion{$10\%$} of the
function range). This visual comparison shows that the main features of the
diagrams (encoding the main hurricane wind gusts in the data) are well
preserved by the data reduction, especially for the members coming from the
cluster $2$, for which a lower \julienQuestion{relative} reconstruction
error (\emph{Err}) can be observed. The planar overview of the ensemble (right)
generated by our dimensionality reduction
(\autoref{sec_dimensionalityReduction}) enables
the
visualization of the relations between the different diagrams of the ensemble.
Specifically, this illustration shows a larger disparity for two clusters.
% ($0$
% and $2$).
% (spread out blue and dark purple spheres), which are also the most difficult to
% reconstruct.
% For this specific figure,
% % we only considered
% only the saddle-maximum
% persistence pairs are considered.
% All other experiments
% (\autoref{Tab:Compress}, \autoref{Tab:aggIndice}, \julienQuestion{Appendix C})
% include all persistence pairs for all ensembles.
}}
% in contrast to the other clusters, which are denser.}}
%  \caption{Visual comparison between original diagrams and their reconstruction.
%  Here we have three terrain representations of three scalar fields taken from
% the Isabel storm dataset, one for each cluster in the set (the left one from
% the first cluster, the middle one from the second and the right one from the
% last). On the bottom left of each of them we have the corresponding original
% diagram and on the bottom right we have the reconstructed diagrams by the
% method. Then we have a 2D projection of all the reconstructed diagrams from the
% atoms and the barycentric weights. The three reconstructions in this example are
% designated as shown, and we show that they are indeed finely clustered. Some
% points can be raised. Firstly we can see that the first diagram is the least
% well approximated as the second main feature in the diagram is not appearing
% clearly in the reconstruction, more precisely the whole first cluster is the
% least well approximated compared to the other two. Secondly we can see that the
% added features in the atoms create artificial noise in the approximations.}
 \label{fig:isabVis}
\end{figure*}

\begin{figure*}[tb]
 \centering
 \includegraphics[width=\linewidth]{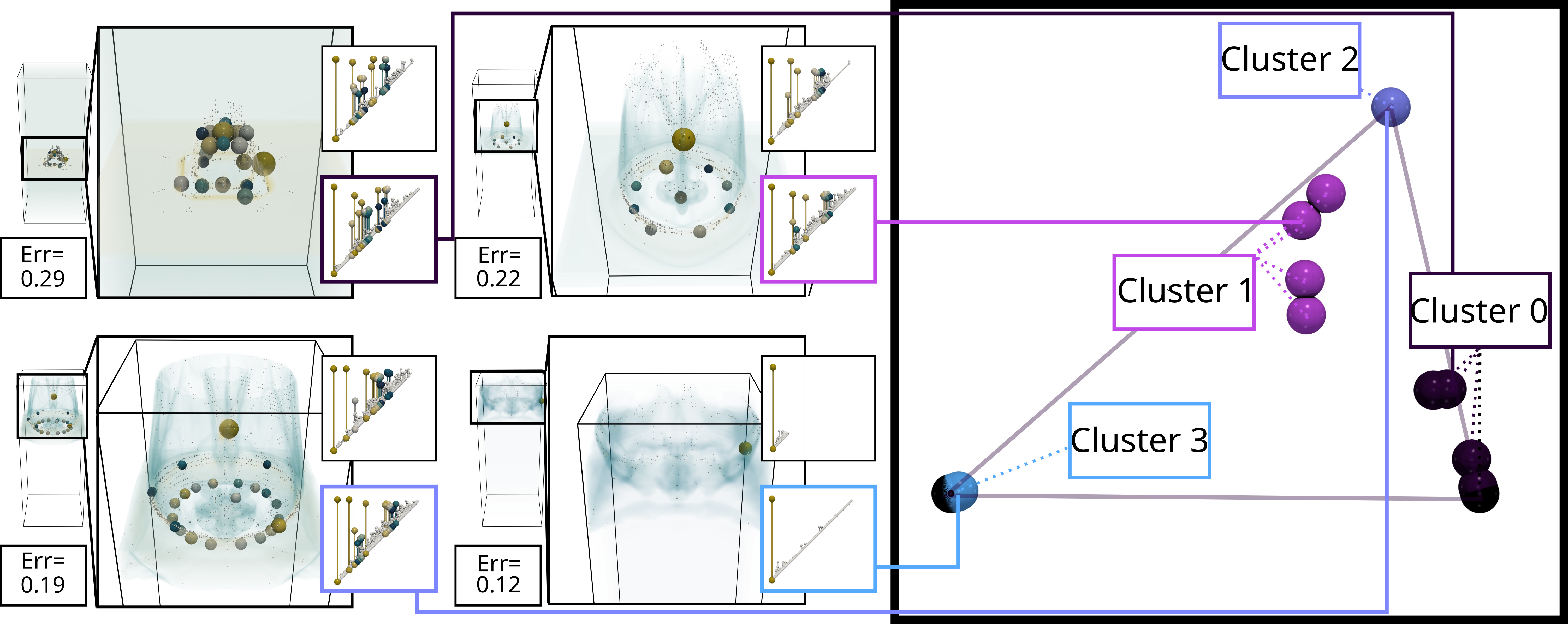}
 \caption{\julien{Visual comparison (left) between the input persistence
diagrams (top insets) and our compressed diagrams (bottom insets, \autoref{sec:
dataRed}) for four members of the \emph{Ionization front (3D)} ensemble (one
member per
ground-truth class). The color encoding is the same as in \autoref{fig:isabVis}.
% For each member, the sphere color encodes
% the matching between
% % the features of
% the input and the compressed diagrams (for
% the meaningful persistence pairs, above \julienQuestion{$10\%$} of the
% function range).
This visual comparison shows that the main features of the
diagrams (the extremities of the ionization front)
% encoding the main hurricane wind gusts in the data)
are well
preserved by the data reduction, especially for the members coming from the
clusters $2$ and $3$, for which a lower
% \julienQuestion{relative}
reconstruction
error (\emph{Err}) can be observed. The planar overview of the ensemble (right)
generated by our dimensionality reduction
(\autoref{sec_dimensionalityReduction}) enables
the
visualization of the relations between the
% different
diagrams of the ensemble.
Specifically,
% this illustration
it shows a larger disparity for the clusters $0$
and $1$
(spread out purple and pink spheres), which are also the most difficult to
reconstruct.}}
%  \caption{Visual comparison between original diagrams and their reconstruction
%  for Ionization front 3D. In this example, we have four terrain representations
% of four scalar fields taken from the Ionization Front 3D dataset, one for each
% cluster in the set (the upper left one from the first cluster, the upper right
% one from the second, the lower left one from the third and the lower right one
% from the last). On the top right of each of them we have the corresponding
% original diagram and the bottom right we have the approximation by the method.
% Then we have a 2D layout of all the reconstructed diagrams from the atoms and
% the barycentric weights.}
 \label{fig:ion3DVis}
\end{figure*}

% This section presents application of our framework in concrete visualization
% tasks: data reduction and dimensionality reduction.

\subsection{Data reduction} \label{sec: dataRed}
\julien{Like any data representation, persistence diagrams can benefit from
lossy compression. This can be useful in in-situ \cite{AyachitBGOMFM15}
use-cases, where time-steps are represented on permanent storage with
topological signatures \cite{BrownNPGMZFVGBT21}. In such scenarios, lossy
compression is useful to facilitate the manipulation (i.e. storage and
transfer) of the resulting ensemble of persistence diagrams.}
% Persistence diagrams can benefit from lossy compression to facilitate their
% storage. In cases where the scalar fields are too big to be stored permanently
% but are represented individually by a topological representation (e.g. a
% persistence diagram, merge tree or Reeb graph), those representations can also
% be quite heavy. This can happen during large data acquisition campaigns, or
% large-scale simulations, where topological representations are computed in-situ
% \cite{AyachitBGOMFM15} to represent each time step \cite{BrownNPGMZFVGBT21}. In
% this case, lossy compression can facilitate the manipulation of the resulting
% ensemble of persistence diagrams.
% In our framework, given the barycentric weights $\Lamb_n$ and $\Dict$ a
% dictionary of persistence diagrams, each $X_n$ the input persistence diagrams
% can be approximated with \autoref{eqn: generalLoss}.
We present now an
application to data reduction where the input ensemble of persistence diagrams 
is compressed, by only storing to disk:

\noindent
\textit{(i)} the \julien{Wasserstein} dictionary of
persistence diagrams $\Dict_{*}$ and

\noindent
\textit{(ii)} the $N$ barycentric weights $\julienRevision{\Lamb_1^*},
\hdots, \julienRevision{\Lamb_N^*}$.

The compression
%factor and
\julien{quality can be controlled with two
input parameters \textit{(i)} the number of atoms $m$ and \textit{(ii)} the
maximum $S_m$ of the total size of the atoms (i.e. $\sum_{i = 1}^m
|a_i|$).} % maximum size of
% each atom, \julien{noted $|a|_{max}$}.
\julien{The reconstruction error (given
by the energy $\dictionaryEnergy$, \autoref{eq_dictionaryEnergy}.) will be
minimized for large values of both parameters, while the compression factor
will be maximized for low values. In our data reduction experiments, we set the
number of atoms $m$ to the number of ground-truth classes of each ensemble, as
documented  in the ensemble descriptions \cite{pont_vis21}.
Moreover, we set $S_m$ to $c_f^{-1} \sum_{i = 1}^N |X_n|$, where $c_f$ is a
target compression factor
\julienRevision{and $|X_n|$ is the number of non-diagonal points in the input
diagram $X_n$}
(see \autoref{section:frameQuality} for a
quantitative evaluation).}
% Specifically, during the optimization procedure
% (\autoref{})}

\julien{\autoref{fig:isabVis} (left) provides a visual comparison between the
diagram compressed with this strategy (bottom insets) and the original diagram
(top insets), for three members of the \emph{Isabel} ensemble. This experiment
shows that diagrams can be significantly compressed ($c_f = 5.49$), while
still faithfully encoding the main features of the data. \autoref{fig:ion3DVis}
(left) provides a similar visual comparison for the \emph{Ionization front
(3D)} ensemble ($c_f = 2.9$).}

\julien{We have applied our data reduction approach to topological clustering
\cite{vidal_vis19}, where the main trends within the ensemble are identified by
clustering the ensemble members based on their persistence diagrams.
For the large majority of our test ensembles,
% Specifically,
the outcome of the clustering algorithm \cite{vidal_vis19} was
identical when used with the input diagrams or our compressed diagrams
(\autoref{section:limits} documents a counter-example).
This confirms the viability and utility
of our data reduction scheme.}

% In this paper for compression problem,
% we set the number of atom $m$
% to be equal to the number of clusters in the initial ensemble, the compression
% factor can be controlled manually by the user. Basically, if we note $S_1^N$ the
% global size of the inputs, and $S_1^m$ the global size of the atoms then we
% impose $S_1^m \leq c_f^{-1} S_1^N$, where $c_f$ is the compression factor.
% \autoref{fig:isabVis} illustrate an example of visualization application where
% our method was applied to the persistence diagrams in the Isabel dataset. In
% this case, the analysis obtained on the approximations is the the same as on the
% original diagrams. This shows the utility of our reduction method.

\subsection{Dimensionality reduction}
\label{sec_dimensionalityReduction}

\julien{Our framework can also be used to generate low-dimensional layouts of
the ensemble, for its global visual inspection. Specifically, we generate 2D
planar layouts by using $m = 3$ atoms and by embedding our Wasserstein
dictionary \julienRevision{$\Dict_*$} as a triangle in the plane, such that its
edge lengths are
equal to the Wasserstein distances between the corresponding atoms. Next, each
diagram $X$ of the input ensemble is embedded as a point in
this triangle by
using its barycentric weights \julienRevision{$\Lamb^*$} as barycentric
coordinates.}

\julien{As illustrated in Figs. \ref{fig:isabVis} (right) and
\ref{fig:ion3DVis} (right), our dimensionality reduction provides a planar
overview of the ensemble which groups together diagrams which are close
in terms of Wasserstein distances. Specifically, in both examples, the
ground-truth classification of the ensemble is visually respected: the points
of a given class (same color) indeed form a distinct cluster in the planar
view.}

\section{Results}
\label{sec_results}

This section presents experimental results obtained on a computer with two Xeon
CPUs (3.2 GHz, 2x10 cores, 96GB of RAM). The input persistence diagrams were
computed with \julien{the \emph{Discrete Morse Sandwich} algorithm}
\cite{guillou2023discrete}.
% , and some
% were pre-processed to discard noisy features (persistence simplification
% threshold: 0.05\%, 0.25\%, 1.00\%, 1.15\% or 10.00\% of the global pair for each
% diagrams).
We implemented our approach in C++ (with OpenMP), as
modules for TTK \cite{ttk17}, \cite{ttk19}. Experiments were ran on the
benchmark of public ensembles \cite{ensembleBenchmark} described in
\cite{pont_vis21}, which includes
% a diverse selection of
simulated and acquired
2D and 3D ensembles from previous work and past SciVis contests
\cite{scivisOverall}.
\julienRevision{The considered type of persistence pairs (i.e. the index of the
corresponding critical points, \autoref{section:distDgm}) was selected on
a per-ensemble basis, depending on
the features of interest present in the ensemble. All types of
% persistence
pairs (i.e. minimum-saddle pairs, saddle-saddle pairs and saddle-maximum pairs)
were considered for the following ensembles:
\emph{Cloud processes}, % 2D -> extrema
\emph{Isabel},
\emph{Starting Vortex}, % 2D -> extrema
\emph{Sea Surface Height},  % 2D -> extrema
\emph{Vortex Street}.  % 2D -> extrema
Only the persistence pairs including extrema were considered for the
% following
ensembles
\emph{Ionization front (2D)} and
\emph{Ionization front (3D)}.
Finally, only the persistence pairs containing maxima were considered for the
remaining ensembles:
\emph{Asteroid Impact},
\emph{Dark Matter},
\emph{Earthquake},
\emph{Viscous Fingering},
\emph{Volcanic Eruptions}.
}

\subsection{Time performance}
\label{sec_timings}

\julien{The most computationally expensive part of our approach is the
computation of the $N$ Wasserstein barycenters, for which we use the algorithm
by Vidal et al. \cite{vidal_vis19}.
Each iteration of barycenter optimization
% computation
approximatively requires  $\mathcal{O}(mK^2)$ steps
in practice (where $K$ is the size of the augmented diagrams, cf.
\autoref{sec_metricSpace}). As discussed in \autoref{section:Para},
each barycenter is computed in parallel.
The evaluations of the gradient of the weight energy (\autoref{subsect:weight})
and the atom energy (\autoref{subsect:atom}) both require $\mathcal{O}(NmK)$
steps. As described in Secs. \ref{subsect:weight}
and \ref{subsect:atom}, both evaluations can be run in parallel.}

% \julien{\autoref{Tab:OneCore} evaluates the practical time performance of our
% framework in sequential mode. In particular, it compares our multi-scale
% algorithm (\autoref{section:Prog}) to a naive approach (implementing
% \autoref{sec:dictEnc}, without any multi-scale aspect). This table shows that
% our multi-scale strategy is about twice slower in practice than the naive
% approach. However, as detailed in \autoref{section:frameQuality}, our
% multi-scale strategy provides solutions of superior quality
% (\autoref{Tab:Compress}).}

\julien{\autoref{Tab:SpeedUp} evaluates the practical time performance of our
multi-scale algorithm for the optimization of the Wasserstein dictionary. In
sequential, the runtime is roughly a function of the number of input diagrams
($N$) as well as their average size ($|X|$).
The parallelization of our algorithm
(with $20$ cores) induces a significant speedup (up to 18 for the largest
ensembles), resulting in an average computation time below $5$ minutes, which
we consider to be an acceptable pre-processing time, prior to interactive
exploration. In comparison to the principal geodesic analysis of persistence
diagrams (Tab. 1 of \cite{pont2022principal}),
on a per ensemble basis, our approach is $1.56$ times faster on average (on the
same hardware).}

\begin{table}
\caption{\julien{Running times (in seconds) of our multi-scale algorithm ($1$
and $20$ cores).}}
\resizebox{\linewidth}{!}{
\begin{tabular}{ |p{3cm}||r|r||r|r||r|  }
%  \hline
%  \multicolumn{6}{|c|}{Time performance} \\
 \hline
 Dataset& N&$|X|$ & 1 core & 20 cores & Speedup\\
 \hline
 Asteroid Impact (3D)& 20 & 220
%  & 35 & 259
  & 259 & 35
 &  7.50\\
 Dark matter (3D) & 40 & 216
%  & 188  & 1,323
& 1,323  & 188
 & 7.04\\
 Earthquake (3D)& 12 & 97
%  & 92  & 113
& 113 & 92
 &1.23\\
 Ionization front (3D)& 16 & 757
%  & 595 & 4,230
& 4,230 & 595
 & 7.11\\
 Isabel (3D) & 12 & 1,310
%  & 270 & 1,609
& 1,609 & 270
 & 5.96\\
 Viscous Fingering (3D) & 15 & 158
%  & 49 & 252
& 252 & 49
 &  5.14\\
 Cloud processes (2D) & 12 & 1,176
%  & 64 & 914
& 914 & 64
 & 14.28\\
 Ionization front (2D)& 16  & 186
%  & \keanu{45}  & \keanu{145}
& 145 & 45
 &\keanu{3.22}\\
 Sea surface height (2D) & 48 & 1,567
%  & 792  & 14,587
& 14,587 & 792
 & 18.42\\
 Starting vortex (2D)& 12 & 125
%  & 24 & 140
& 140 & 24
 & 5.83\\
 Vortex street (2D) & 45 & 43
%  & 241  & 1,061
& 1,061 & 241
 &4.40\\
 Volcanic eruptions (2D) & 12  & 860
%  & 706 & 2,798
& 2,798 & 706
 &3.96\\
 \hline
 
\end{tabular}}
\label{Tab:SpeedUp}
\end{table}

\subsection{Framework quality} \label{section:frameQuality}

\begin{table}
\caption{
\julien{Comparison of the average relative reconstruction error (for a common
target compression factor), between
% PD-PGA
% \cite{pont2022principal},
a naive optimization (\autoref{sec:dictEnc}) and
our multi-scale strategy (\autoref{section:Prog}).
Our multi-scale algorithm
% The latter
improves the
% reconstruction
error by $30\%$ on
average
over the naive approach.}}
\resizebox{\linewidth}{!}{
\begin{tabular}{ |p{3cm}||r|r||r|r||r||r||r|}
%  \hline
%  \multicolumn{11}{|c|}{Error and com pression factor} \\
 \hline
 \textbf{Dataset} & N & $|X|$ & m & $|a|$
%  & \multicolumn{2}{c||}{PD-PGA}
 & Factor
 & \multicolumn{1}{c||}{Error (Naive)} &
\multicolumn{1}{c|}{Error (Multi-Scale)}  \\
%                   &   &                 &   &
%                   & Error & Factor
%                   & Error & Factor& Error & Factor\\
\hline
 Asteroid Impact (3D)& 20 & 220 & 4 & 493 & 2.20
%  & 0.02 & 2.19
 & 0.09
%  & 2.20
 & 0.06 \\
%  & 2.20  \\
 Dark matter (3D) & 40 & 216 & 4 & 215 & 10.87
%  & 0.04 & 10.00
 & 0.15
%  & 10.87
 & 0.12 \\
%  & 10.87  \\
 Earthquake (3D)& 12 &98  & 3 & 120 & 3.05
%  & 0.04 & 2.5
 & 0.16
%  & 2.26
 & 0.04 \\
%  & 3.05  \\
 Ionization front (3D)& 16 & 757 & 4 & 1,044 & 2.90
%  & 0.17 & 3.27
 & 0.29
%  & 3.33
 & 0.20 \\
%  & 2.9 \\
 Isabel (3D) & 12 & 1,310 & 3 & 1,049 & 5.49
%  & 0.27 & 5.49
 & 0.34
%  & 5.50
 & 0.37 \\
%  & 5.49  \\
 Viscous Fingering (3D) & 15 & 158 & 3 & 41 & 2.78
%  &0.13 & 2.23
 & 0.15
%  & 2.50
 & 0.11 \\
%  & 2.78  \\
 Cloud processes (2D) & 12 & 1,176& 3 & 381 & 5.97
%  & 0.19 & 5.94
 & 0.38
%  & 6
 &0.41 \\
%  & 5.97  \\
 Ionization front (2D)& 16 & 186 &  4 & 300 & 2.68
%  & 0.14 & 2.56
 & 0.38
%  & 2.70
 & 0.17 \\
%  & 2.68  \\
 Sea surface height (2D) & 48& 1,567 & 4& 534 & 20.98
%  & 0.18 & 19.59
 & 0.54
%  & 20.81
 & 0.61 \\
%  & 20.98  \\
 Starting vortex (2D)& 12 & 125 & 2 & 379 & 1.98
%  & 0.07 & 1.76
 & 0.22
%  & 2.16
 & 0.09 \\
%  & 1.98  \\
 Vortex street (2D) & 45 & 43 & 5 & 75 & 5.08
%  & 0.03 & 4.79
 & 0.18
%  & 4.88
 & 0.04 \\
%  & 5.08  \\
 Volcanic eruptions (2D) & 12&860 & 3 & 345 & 9.97
%  & 0.12 &9.99
 & 0.20
%  & 9.45
 & 0.20\\
%  & 9.97 \\
 \hline
\end{tabular}}
\label{Tab:Compress}
\end{table}

\begin{figure}
 \centering % avoid the use of \begin{center}...\end{center} and use \centering
% instead (more compact)
 \includegraphics[width=\linewidth]{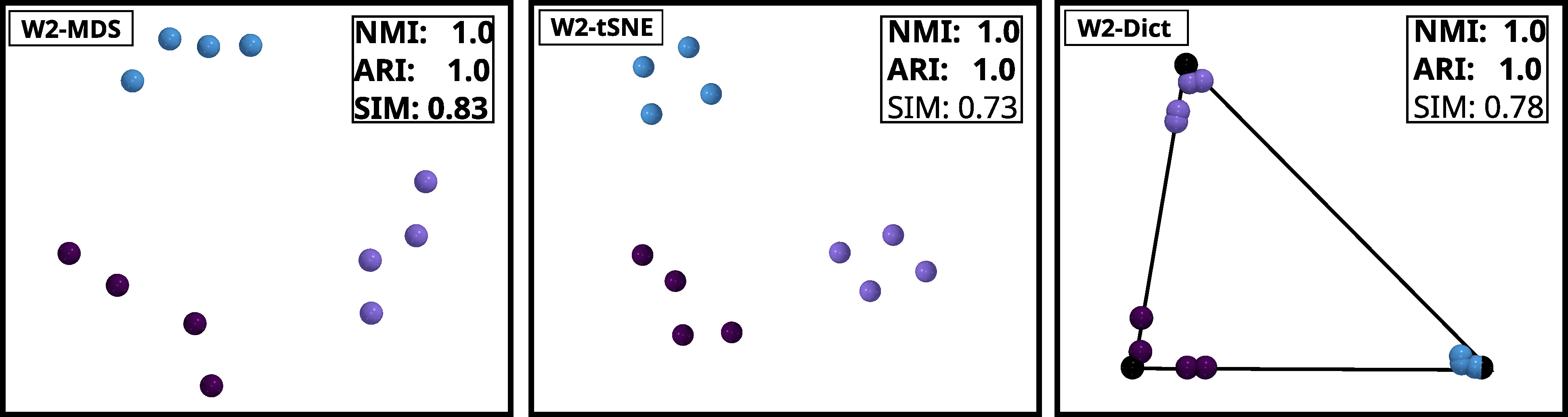}
 \caption{\julien{Comparison between the 2D layouts obtained with our
approach (\emph{W2-Dict}) and these obtained with typical dimensionality
reduction approaches (\emph{W2-MDS} \cite{kruskal78}, \emph{W2-tSNE}
\cite{tSNE}) on the \emph{Isabel} ensemble (all persistence pairs are
considered). Here, the three approaches preserve well the clusters
of the ensemble (NMI/ARI). As expected, \emph{W2-MDS} provides (by design) the
best metric preservation (SIM, bold). Our approach constitutes a trade-off
between \emph{W2-MDS} and \emph{W2-tSNE}.}}
%  \caption{Visual comparison of dimensional reduction methods on the diagrams
%  generated from the Isabel storm dataset. On the left we have the result of our
% method, on the middle we have t-SNE's result, and on the right we have MDS'
% one. In this case we can distinguish 3 clusters of 4 points in the left
% rectangle. }
 \label{fig:isabDim}
\end{figure}

\begin{figure*}
 \centering % avoid the use of \begin{center}...\end{center} and use \centering
% instead (more compact)
 \includegraphics[width=\linewidth]{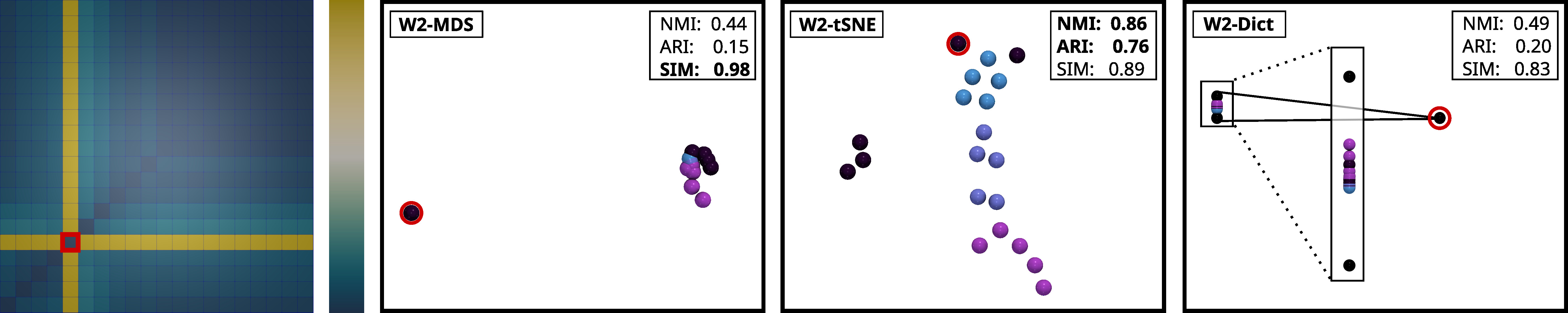}
 \caption{\julien{Comparison between the 2D layouts obtained with our approach
(\emph{W2-Dict}) and these obtained with typical dimensionality reduction
approaches (\emph{W2-MDS} \cite{kruskal78}, \emph{W2-tSNE} \cite{tSNE}) on a
\emph{challenging} ensemble. In this example (\emph{Asteroid Impact}), the
presence of an outlier (time step of the actual impact, red entry in the
distance matrix, left) challenges
cluster preservation. While \emph{W2-tSNE} provides the best cluster
preservation scores (NMI/ARI), it fails at visually depicting the outlier (red
circle) as being far away from the other ensemble members. In contrast,
\emph{W2-MDS} and \emph{W2-Dict} do a better job at isolating this outlier (red
circle), with \emph{W2-Dict} providing slightly improved cluster preservation
scores (NMI/ARI).}}
%  \caption{Visual comparison of dimensional reduction methods on the diagrams
%  generated from the Asteroid Impact dataset. First we have a distance matrix
% between the original diagrams. Then on the left we have the result of our
% method, on the middle we have MDS' result, and on the right we have t-SNE's
% one.
% In this case we can see both our method and MDS struggle to have a good
% clustering, though ours is fairing a bit better. We can also see an outlier in
% the original ensemble as showed in the distance matrix. This outlier can be
% tracked in the 2D planar view and we see it clearly in both our method and MDS.
% }
 \label{fig:asteroidDim}
\end{figure*}

\begin{table*}
\caption{\julienRevision{Detailed layout quality scores
(i.e. bold: best values).
% (i.e.
% averaged over \emph{all} ensembles,
% (bold: best values).
% Overall, o
On average (bottom row), our approach
(\emph{W2-Dict}) provides a trade-off between \emph{W2-MDS} and
\emph{W2-tSNE}: it preserves the clusters (NMI/ARI) slightly better than \emph{W2-MDS}
and % it preserves
the metric (SIM) clearly better than \emph{W2-tSNE}.}}
% \caption{Comparison of 2D projection quality scores, averaged over all data
% ensembles.}
\centering
\begin{tabular}{|p{3cm}||r|r|r||r|r|r||r|r|r|}
\hline
  & \multicolumn{3}{c||}{NMI} & \multicolumn{3}{c||}{ARI} & \multicolumn{3}{c|}{SIM} \\
\hline
Dataset & W2-MDS & W2-tSNE & W2-Dict & W2-MDS & W2-tSNE & W2-Dict & W2-MDS & W2-tSNE & W2-Dict \\
\hline
Asteroid Impact (3D) & 0.44 & \textbf{0.86} & 0.49 & 0.15 & \textbf{0.76} & 0.20 & \textbf{0.91} & 0.89 & 0.83 \\
Dark Matter (3D) & \textbf{1.00} & \textbf{1.00} & \textbf{1.00} & \textbf{1.00} & \textbf{1.00} & \textbf{1.00} & \textbf{0.91} & 0.68 & 0.84 \\
Earthquake (3D) & \textbf{0.65} & 0.61 & \textbf{0.65} & 0.37 & \textbf{0.44} & 0.37 & \textbf{0.96} & 0.72 & 0.91 \\
Ionization Front (3D) & \textbf{1.00} & \textbf{1.00} & \textbf{1.00} & \textbf{1.00} & \textbf{1.00} & \textbf{1.00} & \textbf{0.86} & 0.71 & 0.71 \\
Isabel (3D) & \textbf{1.00} & \textbf{1.00} & \textbf{1.00} & \textbf{1.00} & \textbf{1.00} & \textbf{1.00} & \textbf{0.83} & 0.73 & 0.78 \\
Viscous Fingering (3D) & \textbf{1.00} & \textbf{1.00} & \textbf{1.00} & \textbf{1.00} & \textbf{1.00} & \textbf{1.00} & \textbf{0.91} & 0.64 & 0.89 \\
Cloud Processes (2D) & \textbf{1.00} & \textbf{1.00} & \textbf{1.00} & \textbf{1.00} & \textbf{1.00} & \textbf{1.00} &\textbf{ 0.79} & 0.55 & 0.68 \\
Ionization Front (2D) & \textbf{1.00} & \textbf{1.00} & \textbf{1.00} & \textbf{1.00} & \textbf{1.00} & \textbf{1.00} & 0.78 & 0.74 & \textbf{0.83} \\
Sea Surface Height (2D) & \textbf{1.00} & \textbf{1.00} & \textbf{1.00} & \textbf{1.00} & \textbf{1.00} & \textbf{1.00} & \textbf{0.85} & 0.73 & 0.79 \\
Starting Vortex (2D) & \textbf{1.00} & \textbf{1.00} & \textbf{1.00} & \textbf{1.00} & \textbf{1.00} & \textbf{1.00} & \textbf{0.88} & 0.72 & 0.84 \\
Street Vortex (2D) & \textbf{1.00} & 0.14 & \textbf{1.00} & \textbf{1.00} & -2e-4 & \textbf{1.00} & 0.89 & \textbf{0.96} & 0.81 \\
Volcanic Eruption (2D) & 0.66 & \textbf{1.00} & 0.66 & 0.41 & \textbf{1.00} & 0.41 & \textbf{0.81} & 0.74 & 0.74\\
\hline
Average &  0.896    &0.884 &  \textbf{0.900} & 0.827  & \textbf{0.849}  & 0.832 & \textbf{0.870} & 0.734&  0.804\\
\hline
\end{tabular}
\label{Tab:aggIndice}
\end{table*}

\julien{\autoref{Tab:Compress} reports compression factors and average relative
reconstruction errors for our application to data reduction (\autoref{sec:
dataRed}).
For each ensemble, the compression factor is the ratio between the storage size
of the input diagrams and that of the Wasserstein dictionary
\julienRevision{$\Dict_*$} (the $m$
atoms, of average size $|a|$, plus the $N$ sets of barycentric weights).
The relative reconstruction error is obtained by considering the
Wasserstein distance between an input diagram and its barycentric
approximation, divided by the maximum pairwise Wasserstein distance observed in
the input ensemble. Then this relative reconstruction error is averaged over
all the diagrams of the ensemble. \autoref{Tab:Compress} compares a naive
% implementation of
optimization
(\autoref{sec:dictEnc}) to our multi-scale strategy
(\autoref{section:Prog}). Specifically,
for a given ensemble,
the same target compression factor was
used for both approaches (by imposing the same upper boundary on the total size
of the atoms, \autoref{sec: dataRed}). \autoref{Tab:Compress} shows that our
multi-scale strategy (\autoref{section:Prog}) enables the optimization to
progress towards better solutions, as assessed by the improvement in
reconstruction error of $30\%$ on average.
In comparison to the principal geodesic analysis of persistence diagrams
(Appendix D of \cite{pont2022principal}), for the same compression factors, the
error induced by our approach is on average $1.79$ times larger. However, our
approach is simpler, more flexible (our optimization is not subject to
restrictive constraints, such as geodesic orthogonality) and slightly faster
(\autoref{sec_timings}).}

% In \autoref{Tab:Compress} we have the compression factors and the errors for
% our application to data reductions. These are factors between the size of the
% inputs diagrams and the outputs (atoms and barycentric weights). This factor
% $c_f$ can be forced by controlling the size (number of off diagonal pairs) of
% the atoms, if $K_{\Ens}$ is the global size of the inputs and $K_{\Dict}$ is the
% global size of the atoms, then we delete the least persistent pairs in each
% $a_i$ such that $K_{\Dict} \leq c_f K_{\Ens}$.
% We can see in \autoref{Tab:Compress} that Pont et al's PD-PGA
% \cite{pont2022principal} provides globally better approximation errors with the
% same compression factor. But W2-Dict has slightly better running times as the
% longest running time is at most 800 seconds compared to 1,360 seconds for PD-PGA
% \cite{pont2022principal}. Moreover \autoref{Tab:Compress} shows that on average,
% W2-Dict Prog has a better approximation error that W2-Dict Naive; this entails
% that the progressive method is generally preferable despite having a longer
% running time as indicated in \autoref{Tab:OneCore}.

\julien{\autoref{fig:isabDim} provides a visual comparison between the 2D
layouts obtained with our approach on the \emph{Isabel} ensemble and those
obtained with two typical dimensionality reduction techniques, namely MDS
\cite{kruskal78} and
tSNE \cite{tSNE}, directly applied on the distance matrix obtained by
computing the Wasserstein distance between all the pairs of input diagrams. For
a given technique, to quantify its ability to preserve the \emph{structure} of
the ensemble, we run $k$-means in the 2D layouts and evaluate the quality of
the resulting clustering (given the ground-truth \cite{pont_vis21}) with the
normalized mutual information (NMI) and adjusted rand index (ARI). To quantify
its ability to preserve the \emph{geometry} of the ensemble, we report the
metric similarity indicator SIM \cite{pont2022principal}, which evaluates the
preservation of the Wasserstein metric in the 2D layout. All these scores vary
between $0$ and $1$, with $1$ being optimal.
In \autoref{fig:isabDim},
% this example,
the three
approaches preserve well the clusters of the ensemble (NMI/ARI) and our
approach provides a trade-off between MDS and tSNE in terms of metric
preservation (SIM).
\autoref{fig:asteroidDim} provides another visual comparison on a
\emph{challenging} ensemble (\emph{Asteroid Impact}). There, the presence of an
outlier (time step of the actual impact) challenges cluster preservation.
While tSNE provides the best cluster
preservation
% scores
(NMI/ARI), it fails at visually depicting the outlier (red
circle) as being far away from the other ensemble members. In contrast,
\emph{MDS} and our approach do
% a better job at isolating
isolate
this outlier (red
circle), with our approach providing slightly improved cluster preservation
 (NMI/ARI) over \emph{MDS}. This illustrates the viability of our dimensionality
reductions
% for the visual detection of outliers.
for outlier detection.
\julienQuestion{Appendix C} extends this visual analysis to
all our test ensembles.}

% \begin{table}
% \caption{\julien{\emph{Aggregated} layout quality scores
% (i.e. averaged over \emph{all} ensembles, bold: best values).
% % (i.e.
% % averaged over \emph{all} ensembles,
% % (bold: best values).
% % Overall, o
% Our approach
% (\emph{W2-Dict}) provides a trade-off between \emph{W2-MDS} and
% \emph{W2-tSNE}: it preserves the clusters (NMI/ARI) better than \emph{W2-MDS}
% and
% % it preserves
% the metric (SIM) better than \emph{W2-tSNE}.}}
% % \caption{Comparison of 2D projection quality scores, averaged over all data
% % ensembles.}
% \resizebox{\linewidth}{!}{
% \begin{tabular}{ |p{3cm}||r|r|r|  }
% %  \hline
% %  \multicolumn{4}{|c|}{Dimension reduction quality} \\
%  \hline
%  Indicator& W2-MDS & W2-tSNE & W2-Dict\\
%  \hline
%  NMI  & 0.896    &0.884 &  \textbf{0.900}\\
%  ARI & 0.827  & \textbf{0.849}  & 0.832\\
%  \hline
%  SIM & \textbf{0.870} & 0.734&  0.804\\
%  \hline
% \end{tabular}}
% \label{Tab:aggIndice}
% \end{table}

\julien{\autoref{Tab:aggIndice} extends our quantitative analysis to all our
% test
ensembles.
% By design,
MDS preserves well the metric (high SIM), at
the expense of mixing ground-truth classes (low NMI/ARI). tSNE behaves
symmetrically (higher NMI/ARI, lower SIM).
% \autoref{Tab:aggIndice} shows that
Our approach provides a trade-off between the extreme behaviors of MDS and
tSNE, with a cluster preservation slightly improved over MDS (NMI/ARI), and a
clearly improved
metric preservation over tSNE (SIM).}

% In \autoref{fig:isabDim}, we have
% a comparison of 2D planar layouts between our method, MDS \cite{kruskal78} and
% t-SNE \cite{tSNE} on the Isabel dataset. To quantify the structure
% preservation, we apply the k-means on the 2D points and evaluate the quality of
% the clustering, with respect to the ground-truth \cite{ensembleBenchmark}, with
% the normalized mutual information (NMI) and adjusted rand index (ARI). We also
% introduce a metric similarity indicator, SIM, evaluating the metric preservation
% by a 2D projection of the Wasserstein metric $W_2$. \autoref{fig:asteroidDim}
% shows a case where our method fair a bit better than MDS on the NMI and ARI
% indicator.
% \julien{TODO: Discuss the usage for outlier detection.}
% Appendix A generalizes this analysis to all the datasets.
% \autoref{Tab:aggIndice} presents the aggregated quality indicator results for
% our method, MDS  and t-SNE.

\julien{\autoref{tikz:naiveConvRate} reports the evolution of the normalized
energy $\dictionaryEnergy$ along the optimization for all test ensembles,
% (normalized on a per ensemble basis, based on its initial value)
for the naive
optimization strategy (\autoref{sec:dictEnc}), by using a number of atoms equal
to the number of ground-truth classes (cf. our application to data reduction,
\autoref{sec: dataRed}). In this figure, the energy is normalized on a per
ensemble basis, based on its initial value.
This figure shows that the energy does
decrease for most ensembles,
but still with large oscillations due to the
non-convex nature of the dictionary energy $\dictionaryEnergy$.
In contrast, the energy evolution with our multi-scale strategy
(\autoref{tikz:progConvRate}) results in much less oscillations, which
indicates the ability of this strategy to help the optimization
explore in a more stable manner the locally convex areas of the energy
\julienRevision{(Appendix D discusses a
counter-example)}.
Specifically,
in \autoref{tikz:progConvRate}, one can observe sequences of
discontinuous decrease patterns, characterized by an abrupt drop followed by a
plateau.
Each of these patterns corresponds to one persistence scale of our multi-scale
strategy (this is particularly apparent on the \emph{Cloud Processes}
ensemble).}

\julien{\autoref{tikz:naiveVsProgConvRate} provides a closer comparison between
the two
% optimization
strategies on a selection of four ensembles.
The \emph{Cloud
Processes} ensemble is an example where the naive optimization reaches a
solution of slightly lower energy. For the other ensembles, our multi-scale
strategy leads to solutions of much lower energy, visually confirming the
conclusions of \autoref{Tab:Compress}. In this figure, one can also observe
the
% sequences of
characteristic decrease patterns discussed above, particularly
apparent on the \emph{Ionization Front (3D)} ensemble, which correspond to the
distinct
% persistence
scales of our multi-scale strategy.}

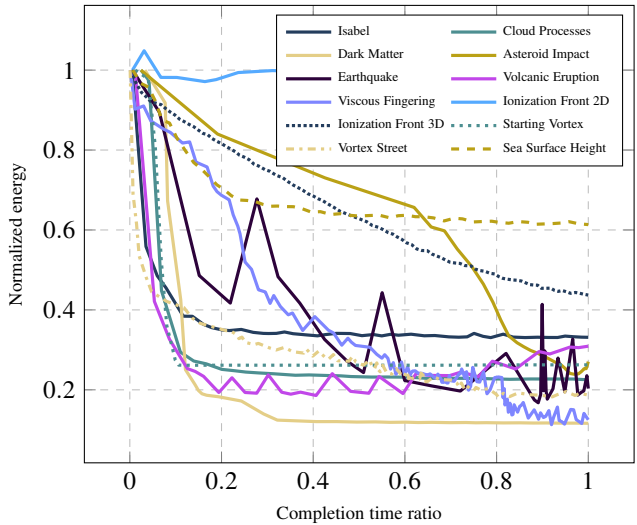
\begin{figure}
\centering
\begin{tikzpicture}
\begin{axis}[group style={group name=plots,},xlabel={Completion time ratio},
ylabel={Normalized energy}, legend cell align={left}]
\addplot[curve1] table [x=Timers, y=Loss, col sep=comma] {Result/PlotIsabLossNaiv.csv};
\addlegendentry{Isabel}
\addplot[curve2] table [x=Timers, y=Loss, col sep=comma] {Result/PlotCloudLossNaiv.csv};
\addlegendentry{Cloud Processes}
\addplot[curve3] table [x=Timers, y=Loss, col sep=comma] {Result/PlotDarkMattLossNaiv.csv};
\addlegendentry{Dark Matter}
\addplot[curve4] table [x=Timers, y=Loss, col sep=comma] {Result/PlotAsteroidLossNaiv.csv};
\addlegendentry{Asteroid Impact}
\addplot[curve5] table [x=Timers, y=Loss, col sep=comma] {Result/PlotEarthQuakeLossNaiv.csv};
\addlegendentry{Earthquake}
\addplot[curve6] table [x=Timers, y=Loss, col sep=comma] {Result/PlotVolcanoLossNaiv.csv};
\addlegendentry{Volcanic Eruption}
\addplot[curve7] table [x=Timers, y=Loss, col sep=comma] {Result/PlotViscousLossNaiv.csv};
\addlegendentry{Viscous Fingering}
\addplot[curve8] table [x=Timers, y=Loss, col sep=comma] {Result/PlotIon2DLossNaiv.csv};
\addlegendentry{Ionization Front 2D}
\addplot[curve9] table [x=Timers, y=Loss, col sep=comma] {Result/PlotIon3DLossNaiv.csv};
\addlegendentry{Ionization Front 3D}
\addplot[curve10] table [x=Timers, y=Loss, col sep=comma] {Result/PlotStartLossNaiv.csv};
\addlegendentry{Starting Vortex}
\addplot[curve11] table [x=Timers, y=Loss, col sep=comma] {Result/PlotStreetLossNaiv.csv};
\addlegendentry{Vortex Street}
\addplot[curve12] table [x=Timers, y=Loss, col sep=comma] {Result/PlotSeaSurfLossNaiv.csv};
\addlegendentry{Sea Surface Height}
\end{axis}
\end{tikzpicture}
\caption{\julien{Evolution of the (normalized) energy $\dictionaryEnergy$
along the optimization, with a naive optimization (\autoref{sec:dictEnc}),
for all our test ensembles.}}
% \caption{Convergence curves of the normalized loss optimized for the naive
% approach.}
\label{tikz:naiveConvRate}
\end{figure}

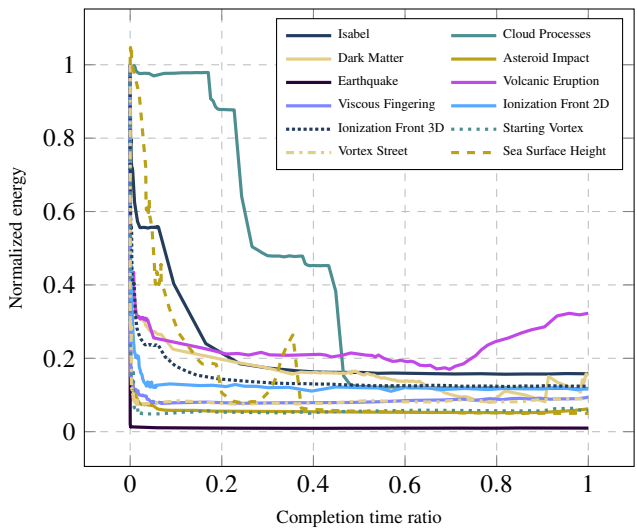
\begin{figure}
\centering
\begin{tikzpicture}
\begin{axis}[group style={group name=plots,},xlabel={Completion time ratio},
ylabel={Normalized energy}, legend cell align={left}]
\addplot[curve1] table [x=Timers, y=Loss, col sep=comma] {Result/PlotIsabLoss.csv};
\addlegendentry{Isabel}
\addplot[curve2] table [x=Timers, y=Loss, col sep=comma] {Result/PlotCloudLoss.csv};
\addlegendentry{Cloud Processes}
\addplot[curve3] table [x=Timers, y=Loss, col sep=comma] {Result/PlotDarkMattLoss.csv};
\addlegendentry{Dark Matter}
\addplot[curve4] table [x=Timers, y=Loss, col sep=comma] {Result/PlotAsteroidLoss.csv};
\addlegendentry{Asteroid Impact}
\addplot[curve5] table [x=Timers, y=Loss, col sep=comma] {Result/PlotEarthQuakeLoss.csv};
\addlegendentry{Earthquake}
\addplot[curve6] table [x=Timers, y=Loss, col sep=comma] {Result/PlotVolcanoLoss.csv};
\addlegendentry{Volcanic Eruption}
\addplot[curve7] table [x=Timers, y=Loss, col sep=comma] {Result/PlotViscousLoss.csv};
\addlegendentry{Viscous Fingering}
\addplot[curve8] table [x=Timers, y=Loss, col sep=comma] {Result/PlotIon2DLoss.csv};
\addlegendentry{Ionization Front 2D}
\addplot[curve9] table [x=Timers, y=Loss, col sep=comma] {Result/PlotIon3DLoss.csv};
\addlegendentry{Ionization Front 3D}
\addplot[curve10] table [x=Timers, y=Loss, col sep=comma] {Result/PlotStartLoss.csv};
\addlegendentry{Starting Vortex}
\addplot[curve11] table [x=Timers, y=Loss, col sep=comma] {Result/PlotStreetLoss.csv};
\addlegendentry{Vortex Street}
\addplot[curve12] table [x=Timers, y=Loss, col sep=comma] {Result/PlotSeaSurfLoss.csv};
\addlegendentry{Sea Surface Height}

\end{axis}
\end{tikzpicture}
\caption{\julien{Evolution of the (normalized) energy $\dictionaryEnergy$
along the optimization, with our multi-scale strategy (\autoref{section:Prog}),
% a naive implementation of \autoref{sec:dictEnc},
for all our test ensembles.}}
% \caption{Convergence curves of the normalized loss optimized for the progressive approach.}
\label{tikz:progConvRate}
\end{figure}

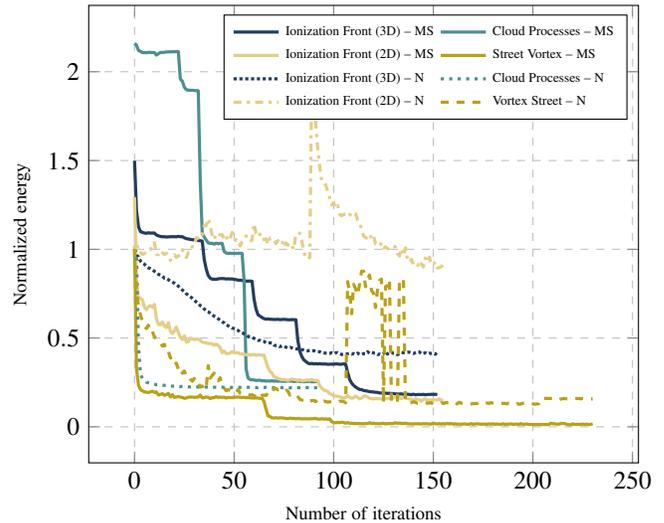
\begin{figure}
\centering
\begin{tikzpicture}
\begin{axis}[group style={group name=plots,},xlabel={Number of iterations},
ylabel={Normalized energy}, legend cell align={left}]
\addplot[curve1] table [x=Iterations, y=Loss, col sep=comma] {Result/PlotIon3DLossProg.csv};
\addlegendentry{Ionization Front (3D) -- MS}
\addplot[curve2] table [x=Iterations, y=Loss, col sep=comma] {Result/PlotCloudLossProg.csv};
\addlegendentry{Cloud Processes -- MS}
\addplot[curve3] table [x=Iterations, y=Loss, col sep=comma] {Result/PlotIon2DLossProg.csv};
\addlegendentry{Ionization Front (2D) -- MS}
\addplot[curve4] table [x=Iterations, y=Loss, col sep=comma] {Result/PlotStreetLossProg.csv};
\addlegendentry{Street Vortex -- MS}
\addplot[curve9] table [x=Iterations, y=Loss, col sep=comma] {Result/PlotIon3DLossNaiv2.csv};
\addlegendentry{Ionization Front (3D) -- N}
\addplot[curve10] table [x=Iterations, y=Loss, col sep=comma] {Result/PlotCloudLossNaiv2.csv};
\addlegendentry{Cloud Processes -- N}
\addplot[curve11] table [x=Iterations, y=Loss, col sep=comma] {Result/PlotIon2DLossNaiv2.csv};
\addlegendentry{Ionization Front (2D) -- N}
\addplot[curve12] table [x=Iterations, y=Loss, col sep=comma] {Result/PlotStreetLossNaiv2.csv};
\addlegendentry{Vortex Street -- N}
\end{axis}
\end{tikzpicture}
% \vspace{-1ex}
\caption{
\julien{Comparison of the evolutions of the (normalized) energy
$\dictionaryEnergy$ between the naive optimization
(\autoref{sec:dictEnc}, \emph{N}, dashed curves) and our multi-scale strategy
(\autoref{section:Prog}, \emph{MS}, solid curves) for four ensembles. For this
experiment, the energy has been normalized with regard to the initial energy
% reached by
of the naive optimization.
% and both approaches ran the same number of iterations.
The \emph{Cloud
Processes} ensemble is an example where the naive optimization reaches a
solution of slightly lower energy. For the other three ensembles, our
multi-scale strategy leads to solutions of much lower energy, through a sequence
of
characteristic,
discontinuous decrease patterns
(abrupt drop followed by a plateau) corresponding to the five persistence
scales of our multi-scale strategy.}}
% \caption{Comparison of both approaches on four datasets: Cloud Processes,
% Ionization Front 2D, Street Vortex and Ionization Front 3D. For each example, we
% normalized the objective energy by the initial energy of the naive approach. For
% each dataset, we force the number of iterations for the naive approach to be
% equal to the progressive one. We see that for Cloud Processes, the naive
% approach is better than the progressive one. But for the other three datasets
% the progressive approach leads us into a better local minimum.}
\label{tikz:naiveVsProgConvRate}
\end{figure}

\begin{figure}[tb]
 \centering % avoid the use of \begin{center}...\end{center} and use \centering instead (more compact)
 \includegraphics[width=\columnwidth]{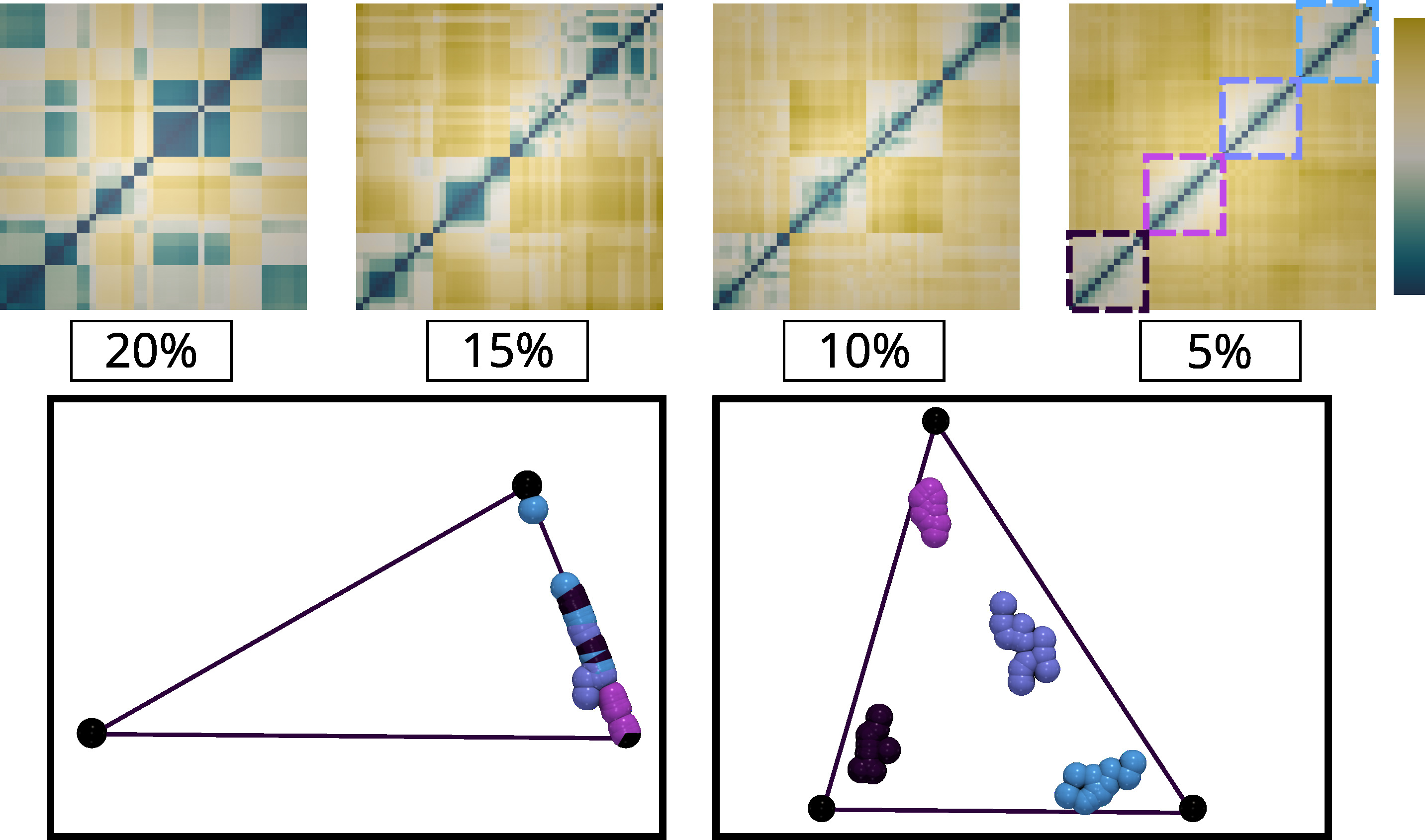}
 \caption{\julien{Counter-example for our multi-scale strategy (\emph{Sea
Surface Height} ensemble). Top: Wasserstein distance matrices for the first
four persistence scales of our multi-scale strategy. The ground-truth classes
only start to become visible in the distance matrix between the third and
fourth scale (dashed sub-matrices in the fourth scale). As a result, our
multi-scale strategy is attracted in the first scales towards a local minimum
of the energy which does not encode well the ground-truth classes
(dimensionality reduction, bottom left). In contrast, the naive optimization
 manages to reach a solution
which separates well the ground-truth classes (dimensionality reduction, bottom
right).}}
%  \caption{At the top we can see 4 distance matrices between the diagrams in
%  the SeaSurfaceHeight dataset were we filtered the pairs that are below 20\%,
% 15\%, 10\% and 5\% of the global pair of each diagram. And below them are the
% result of dimensional reduction of our method, with progressivity on the left
% and without progressivity on the right. We precise that for this example we do
% not add any features in the atoms for the naive approach.}
 \label{fig:progSee}
\end{figure}

\subsection{Limitations} \label{section:limits}
\julien{Similarly to other optimization problems based on topological
descriptors \cite{Turner2014, pont_vis21, pont2022principal, vidal_vis19}, our
energy is not convex.
% Therefore, our optimization will reach the vicinity of a
% minimum which is only local.
Additionally, as shown in \autoref{tikz:naiveConvRate},
the interleaving of the weight optimization (\autoref{subsect:weight}) with
atom optimization (\autoref{subsect:atom}) can even lead to oscillations in the
energy. As discussed in \autoref{section:frameQuality}, our
multi-scale strategy (\autoref{section:Prog}) greatly mitigates both issues,
with a more stable optimization than a naive approach (\autoref{sec:dictEnc}),
which leads to relevant solutions which are exploitable in the applications
(\autoref{sec_applications}). However, we have found
one example in our test ensembles (the \emph{Sea Surface Height} ensemble),
% certain cases
where our
multi-scale strategy reached solutions which were arguably worse than these
obtained with a naive solution, as described in details in
\autoref{fig:progSee}.
% in the case of .
In
this example, the most persistent features in the diagrams are not particularly
discriminative for the separation of the ground-truth classes. On the contrary,
the variations between these classes seem mostly encoded by the
\emph{low} persistence features: in \autoref{fig:progSee} clear separations in
the distance matrices between the ground-truth classes only start to occur in
the latest persistence scales (dashed sub-matrices, top right inset). This
counter-intuitive observation goes against the rule of thumb traditionally used
in topological data analysis, which states that the most persistent pairs encode
the most important features in the data. For this example, when applying our
framework to dimensionality reduction, the non-discriminative aspect of the
early persistence scales eventually lead our multi-scale strategy towards a
local minium which does not separate the ground-truth classes well (planar
layout, bottom left) in comparison to the naive strategy (planar layout, bottom
right).
Thus,
for this ensemble, we reported dimensionality reduction results
(\autoref{Tab:aggIndice}, \julienQuestion{Appendix C}) obtained with the naive
optimization.
In general,
% Generally speaking,
this means that when users are confronted
with ensembles where the most persistent pairs are not the most responsible for
data variability (hence class separation), the naive optimization may need
to be considered additionally as it might provide
% dimensionality reductions
solutions
which better
encode the ground-truth classes.}

\julienRevision{Finally, as detailed in Appendix D, the presence of clear
outliers
% in the input ensemble
can also challenge our optimization, especially
when the selected number of atoms
equals the number of ground-truth classes. Then, in this case, the best
dictionary encoding will consequently be obtained by increasing the number of
atoms, specifically, by considering that each outlier forms a singleton class.}
%
% the number of
% atoms should be increased, considering that each outlier forms a singleton
% class.}

% \julien{- the energy can increase
% - the solution is a local minimum only
% - progressive approach helps
% - own limitations discussed with fig15 -> non-optimal local minima or slower
% convergence.}
%
% We have indicated in \autoref{section:Grad} that the function optimized is
% neither convex nor is continuous. This theoretical problem can lead to an
% increase of the energy when optimizing, and also rendering our algorithm unable
% to converge towards a global minimum if it exists. We globally manage to tackle
% this problem with the progressive approach, as we saw in
% \autoref{tikz:progConvRate}. Another problem is that the progressive approach is
% not entirely foolproof, and we have seen an example in \autoref{fig:progSee}
% where the clustering obtained on the 2D layout was wrong. Another limitation
% aspect is the atoms optimization. As detailed in \autoref{section:Grad}, we
% optimize the position of the pairs in each atoms using the matchings involved in
% the Wasserstein distance and in the barycenters definition. This optimization
% process does not take into account the geometry of the space of persistence
% diagrams.

\section{Conclusion}

\julien{In this paper, we presented an approach for the encoding of linear
relations between persistence diagrams, given the Wasserstein metric.
Specifically, we introduced
a dictionary based representation of an ensemble of persistence diagrams,
inspired by previous work on histograms \cite{schmitz2018wasserstein}. We
first documented a naive optimization,
which interleaves
% interleaving
the optimization of the
barycentric weights of the input diagrams with the optimization of the atoms of
the dictionary (\autoref{sec:dictEnc}). Then, we presented a multi-scale
strategy (\autoref{section:Prog}) leading to more stable optimizations and
relevant solutions (\autoref{section:frameQuality}). We demonstrated the utility
of our contributions in applications (\autoref{sec_applications}) to data
reduction and dimensionality reduction, where the visualizations generated by
our framework enable the visual identification of the main trends in the
ensembles (Figs. \ref{fig:isabVis}, \ref{fig:ion3DVis}), and the
quick identification of outliers (\autoref{fig:asteroidDim}). In contrast to
previous work
on persistence diagram encoding
% on the  encoding of persistence diagrams
\cite{pont2022principal}, our framework is simpler, less constrained and
slightly faster in practice.}

% In this paper, we presented an algorithm for the dictionary encoding of
% persistence diagrams, with applications to data reduction and dimensionality
% reduction. Our approach adapts the Wasserstein Dictionary Learning framework
% \cite{schmitz2018wasserstein} on persistence diagrams, and revisits the
% progressivity strategy on persistence \cite{vidal_vis19} on our framework. The
% visualizations produced by our main contribution (Figs. \ref{fig:isabVis},
% \ref{fig:ion3DVis}, \ref{fig:isabDim}, \ref{fig:asteroidDim}) allow inspection
% of the variability in the ensemble at a global level with our planar layout.

A natural direction for future work is the extension of our framework to other
topological
% data representations,
\julien{descriptors}
such as Reeb graphs or Morse-Smale complexes.
However, this requires the definition of key geometrical tools, such as
% metrics
\julien{geodesic}
or barycenter computation \julien{algorithms}, which is still an active
research problem. We believe our \julien{framework for the dictionary encoding
of persistence diagrams is an interesting practical step for the analysis of
% large-scale
collections of persistence diagrams. In the future, we will
continue our investigation of the adaptation of tools from optimal transport to
the analysis of ensembles of topological descriptors, as we believe it can
become a key solution in the long term for the advanced analysis of
large-scale ensembles.}
% adaptation of Wasserstein Dictionary Learning
% is an important step towards larger sparse dictionary encoding framework on
% persistence diagrams. In the future, we will continue our investigation of the
% adaptation of optimal transport based method to ensembles of topological
% objects.

\section*{Acknowledgments}{\small
This work is partially supported by the European Commission grant
ERC-2019-COG \emph{``TORI''} (ref. 863464, \url{https://erc-tori.github.io/}).}

\bibliographystyle{abbrv}

\bibliography{template}

\begin{IEEEbiography}[
{\includegraphics[width=1in,height=1.25in,clip,
keepaspectratio]{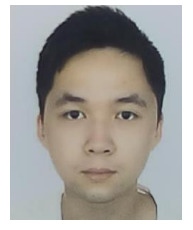}}]{Keanu Sisouk}
is a Ph.D. student at Sorbonne University. He received his master degree in Mathematics from Sorbonne University in 2021. His fields of interests lie on topological methods for data analysis, optimal transport, optimization methods, statistics and partial derivative equations.
\end{IEEEbiography}

\begin{IEEEbiography}[
{\includegraphics[width=1in,height=1.25in,clip,
keepaspectratio]{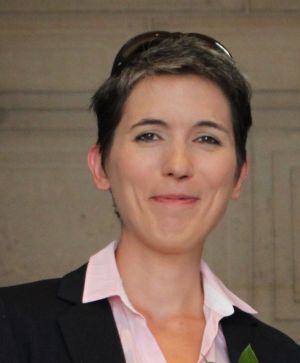}}]{Julie Delon}
%Julie Delon
received the Pd.D. degree in Mathematics from the Ecole Normale Supérieure Cachan in 2004.
She is currently a professor at Paris-Cité University since 2013. Prior to her professor tenure, she was a CNRS researcher affiliated with TELECOM ParisTech.
Her research interests lies in optimal transport, image processing, inverse problems and stochastic models for image restoration and editing.
\end{IEEEbiography}

\begin{IEEEbiography}[
{\includegraphics[width=1in,height=1.25in,clip,
keepaspectratio]{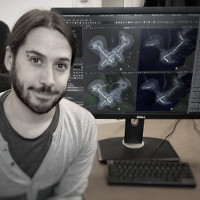}}]{Julien Tierny}
% Julien Tierny
received the Ph.D. degree in Computer Science from the University
of Lille in 2008.
% and the Habilitation degree (HDR) from Sorbonne University in
% 2016. 
He is currently a CNRS research director, affiliated with
Sorbonne University. Prior to his CNRS tenure, he held a
Fulbright fellowship (U.S. Department of State) and was a post-doctoral
researcher at the Scientific Computing and Imaging Institute at the University
of Utah.
His research expertise lies in topological methods for data analysis
and visualization.
% He co-authored two books on the topic and received several
% awards for his research, including best paper awards.
% He regularly serves as an international program committee member for the top
% venues in data visualization (IEEE VIS, EuroVis, etc.) and he is an associate
% editor for IEEE Transactions on Visualization and Computer Graphics.
% Julien Tierny
He is the founder and lead developer of the Topology ToolKit
(TTK), an open source library for topological data analysis.
\end{IEEEbiography}

\cleardoublepage

\title{Appendix}
%% This is how authors are specified in the journal style

%% indicate IEEE Member or Student Member in form indicated below
% \author{Julien Tierny, Julie Delon, Keanu Sisouk}
%\author{Keanu Sisouk, Julie Delon, Julien Tierny}
%
%
%\usepackage[ruled,vlined]{algorithm2e}
%\usepackage{amsmath} 
%\usepackage{amssymb} 
%\usepackage{amsthm}
%\usepackage{bm}
%\usepackage{hyperref}
%\usepackage{subcaption}
%\usepackage{xcolor}

%\input{commands}

%\maketitle

\appendices

\section{\julien{$\nabla \weightEnergy$ is $L$-Lipschitz}}

\begin{Pro}
Let $X$ be a persistence diagram and $\Dict = (a_1,\hdots,a_m)$ a
Wasserstein dictionary of persistence diagrams.
If the optimal matchings  are constant, then
\julien{$\weightEnergy(\Lamb)
% \mapsto
=
\War\big(\BaryLamb ,X\big)$} is convex
% and $\beta$-smooth
\julien{and $\nabla \weightEnergy$ is $L$-Lipschitz}
on $\Sigma_m$.
\end{Pro}

\begin{proof}
Let $\Lamb = (\lambda_1, \hdots , \lambda_m) \in \Sigma_m$, $\BaryLamb = \big(y^1(\Lamb),\hdots, y^K(\Lamb)\big)$ the barycenter computed and $\phi_{\Lamb , 1} , \hdots , \phi_{\Lamb , m}$ the matchings between $\BaryLamb$ and each atom $(a_1 , \ldots , a_m)$:
\begin{equation}
\forall j \in \{1,\hdots, K\},\ y^j(\Lamb) = \sum\limits_{i=1}^{m}\lambda_i a_i^{\phi_{\Lamb,i}(j)}.
\end{equation}
We suppose the optimal matchings to be constant, thus we write $\AssI = \phi_{\Lamb,i}$. Like before we consider the following gradient:

\begin{equation}
\nabla y^j(\Lamb) = \begin{bmatrix}
a_1^{\Ass(j)} & \ldots & a_m^{\AssM(j)}
\end{bmatrix}.
\end{equation}
Now recall the following expression for:
\begin{equation}
\War\big(\BaryLamb ,X\big) = \min_{\psi_{\Lamb} :
% i_{\BaryLamb}
\julien{i_X}
\xrightarrow{bij} i_X} \left( \sum\limits_{j=1}^{K} \| y^j(\Lamb)
- x^{\psi_{\Lamb}(j)} \|^2 \right).
\end{equation}
This minimum is always attained, and with the hypothesis on the optimal matchings we write $\psi = \psi_{\Lamb}$. Thus we rewrite:
\begin{equation}
\War\big(\BaryLamb ,X\big) =  \sum\limits_{j=1}^{K} \| y^j(\Lamb) - x^{\psi(j)} \|^2 = \sum\limits_{j=1}^{K} \| \sum\limits_{i=1}^{m} \lambda_i (a_i^{\AssI(j)} - x^{\psi(j)}) \|^2. 
\end{equation}
\julien{$\War\big(\BaryLamb ,X\big)$ is convex with $\Lamb$ and the} gradient
follows naturally:
\begin{equation}
\nabla \War\big(\BaryLamb , X\big) = 2
% \julien{\sum\limits_{i=1}^{m}}
\sum\limits_{j=1}^{K}
\begin{bmatrix}
(a_1^{\Ass(j)} - x^{\psi(j)})\julien{^T} \\
\vdots \\ 
(a_m^{\AssM(j)} - x^{\psi(j))\julien{^T}}
\end{bmatrix}\big(y^j(\Lamb) - x^{\psi(j)}\big).
\end{equation}
For the following part we denote $H^j = \begin{bmatrix}
a_1^{\Ass(j)} - x^{\psi(j)} & \ldots & a_m^{\AssM(j)} - x^{\psi(j)}
\end{bmatrix}$. The Hessian then writes as
$H = H(\Lamb) = 2 \sum\limits_{j=1}^{K} (H^j)^T H^j.$ This shows that $\Lamb \mapsto \War\big(\BaryLamb ,X\big)$ is convex. Indeed for $u \in \mathbb{R}^m$ we have:
\begin{equation}
u^T H u = 2 \sum\limits_{j=1}^{K} u^T(H^j)^T (H^j) u=  2 \sum\limits_{j=1}^{K} \| H^j u \|^2 \geq 0 
\end{equation}
This also shows that \keanu{$\nabla\weightEnergy $ is $L$-Lipschitz with $L = \|H \|$.} For numerical reasons, we bound \keanu{$L$} as follows:
\begin{equation}
\keanu{L} = \|H \| = 2\left\lVert \displaystyle\sum_{j=1}^{K} (H^j)^T(H^j) \right\rVert \leq  2 \displaystyle\sum_{j=1}^{K} \| (H^j)^T(H^j) \| =  2 \displaystyle\sum_{j=1}^{K} \| H^j \|^2.
\end{equation}
Thus for our algorithm, we consider the following gradient step:
\begin{equation}
\rho \leq \left[2 \displaystyle\sum_{j=1}^{K} \| H^j \|^2\right]^{-1}.
\end{equation}
\end{proof}

~

\section{\julien{$\nabla \individualAtomEnergy$ is $L$-Lipschitz}}

\begin{Pro}
Let $X$ be a persistence diagrams and $\Lamb = (\lambda_1, \hdots, \lambda_m)
\in \Sigma_m$. If the optimal matchings are constant, the functions
$\individualAtomEnergy$
% introduced in Sect. 3.3.2
are convex and \keanu{$\nabla \individualAtomEnergy$
is $L$-Lipschitz.}
\end{Pro}

\begin{proof}
Let $U = (u_1, \hdots , u_m) \in (\mathbb{R}^2)^m$, for $j \in \{1,\hdots,K\}$ we have:
\begin{equation}
\keanu{\individualAtomEnergy(U)} =  \left\lVert \sum\limits_{i=1}^{m} \lambda_i (u_i - x^{\psi(j)}) \right\rVert^2
\end{equation}
The gradient follows naturally:
\begin{equation}
\label{eqn:gradientAtom}
\nabla \keanu{\individualAtomEnergy(U)} = 2 \begin{bmatrix}
\lambda_1\\
\vdots\\
\lambda_m
\end{bmatrix} \big(u_i - x^{\psi(j)}\big)\keanu{^T}
\end{equation}
Immediately we have the Hessian $H_j = H_{g_j}(U) = 2 \Lamb \Lamb^T$, giving us the convexity \keanu{of $\individualAtomEnergy$} and \keanu{the $L$-Lipschitzianity of $\nabla \individualAtomEnergy$ with $L = \|H_j\| \leq 2\|\Lamb\|^{2} \leq 2m$.} For numerical reasons, we consider the larger upper bound: \keanu{$L\leq 4 m$}.
Thus for our algorithm, we consider the following gradient step $\rho \leq (4m)^{-1}$.
\end{proof}

\begin{figure*}[ht]
\centering
\includegraphics[width=\linewidth]{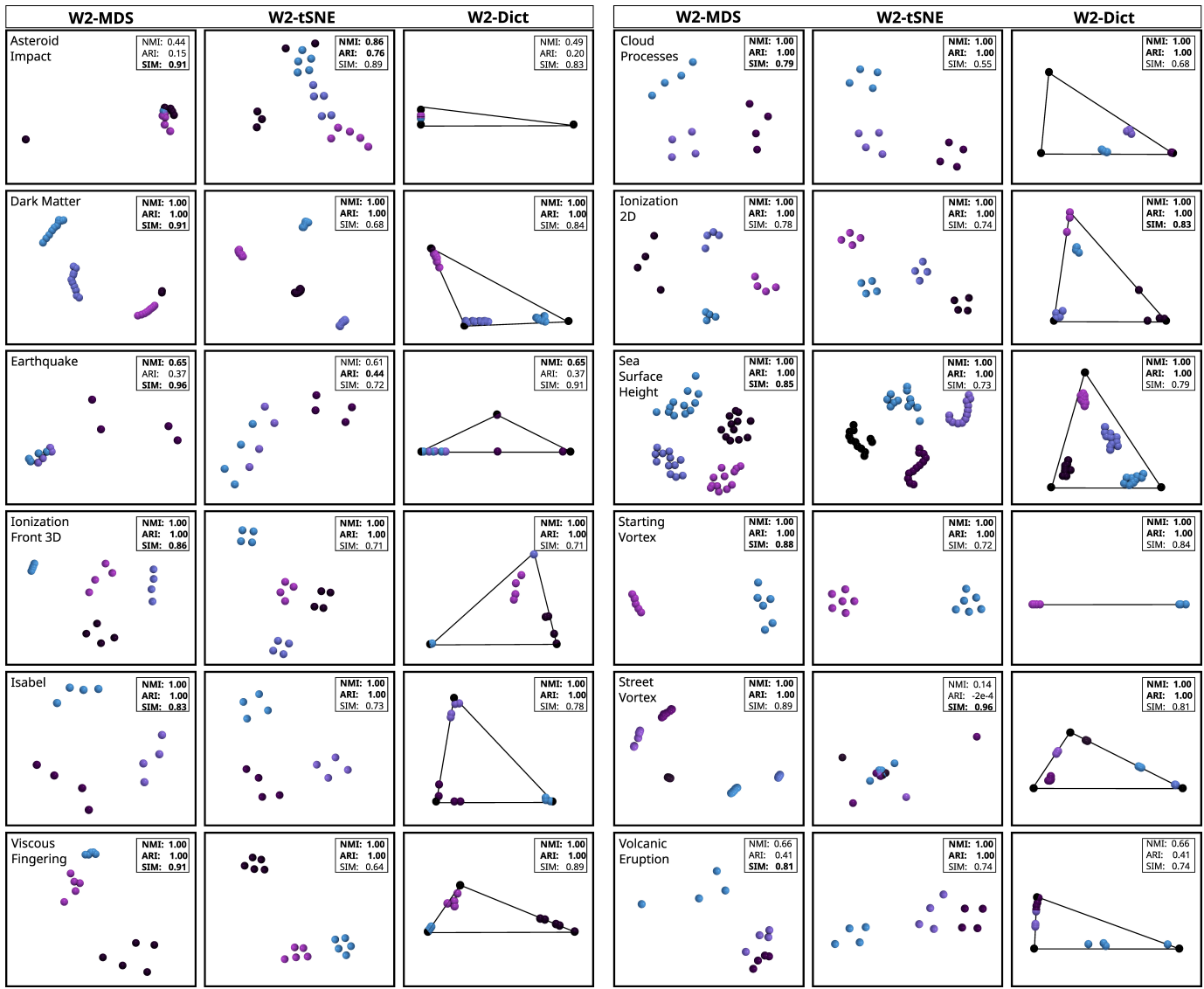}
\caption{\julien{Comparison of the planar layouts for typical dimensionality
reduction techniques on all our test ensembles. The color encodes the
classification ground-truth \cite{pont_vis21}. For each quality score, the
best value appears bold. For the \emph{Sea Surface Height} ensemble, the naive
optimization procedure has been used (cf. Sec. 6.3 of the main manuscript).}}
% \caption{In this figure, there are the comparison between our methods, TSNe and MDS for dimensional reduction on all the 2D data sets considered in this paper in this following order: Cloud processes, Ionization frond 2D, Starting vortex, Sea surface height, Street vortex and Volcanic eruptions}
\label{fig:allDimReduct2D}
\end{figure*}

\section{\julien{Dimensionality reduction}}
\julien{\autoref{fig:allDimReduct2D} extends \julienQuestion{Figure 10} (main
manuscript) to all our test ensembles and it confirms visually the conclusions
of the table of quality scores (\julienQuestion{Table 3} of the main
manuscript).}

% \newpage

~

\section{Volcanic eruption ensemble}
\julienRevision{This appendix discusses the special case of the \emph{Volcanic
eruption} ensemble (12 members), for
which a
consistent energy increase can be observed in
the Figure 13 of
the main manuscript (normalized energy of our multi-scale optimization as a
function of computation time), beyond 70\% of the completion time (the
optimization reaches the stopping conditions at 100\%).}

\julienRevision{The ground-truth classification of this ensemble contains 3
classes \cite{pont_vis21}. However, one of these classes contains a clear
outlier (light purple entry in the bottom views of \autoref{fig:volcano}),
corresponding to a peak of activity in the eruption (see the terrain views of
4 members, bottom left of \autoref{fig:volcano}, including the outlier,
light purple frame). The corresponding persistence diagram (light purple
diagram in the aggregated birth/death space, bottom middle of
\autoref{fig:volcano}) contains features which are significantly more persistent
than the other diagrams (taken from distinct ground-truth classes, one color per
class). Then, this outlier exhibits an excessively high distance to the rest of
the ensemble, as illustrated in the Wassertein distance matrix (bottom right of
\autoref{fig:volcano}, light purple entry).}

\julienRevision{The presence of this outlier challenges our optimization when
using a number of atoms equal to the number of ground-truth classes (which is
the default strategy documented in the main manuscript). As shown in the energy
plots (\autoref{fig:volcano}, top), a consistent energy increase can be
observed when using only 3 atoms (1 per ground-truth class, black curve). When
removing the outlier, the energy evolution exhibits a more
characteristic oscillating behavior (green curve). Finally, when initializing
the optimization with 4 atoms (1 per class, plus 1 for the outlier), the
optimization results in few oscillations and a consistent energy decrease
(yellow curve). This indicates that the outlier member (light purple) should be
interpreted as a singleton class and that the best dictionary encoding will
consequently be obtained with 4 atoms.}

\begin{figure}
\centering
\subfloat{
\begin{tikzpicture}
\begin{axis}[group style={group name=plots,},xlabel={Completion time ratio},
ylabel={Normalized energy}, legend cell align={left}]
\addplot[curve1] table [x=Timers, y=Loss, col sep=comma] {Appendix/PlotVolcanoLoss.csv};
% \addlegendentry{Volcanic eruption using heuristic initialization}
\addlegendentry{3 atoms, 12 members.}
\addplot[curve2] table [x=Timers, y=Loss, col sep=comma] {Appendix/PlotVolcanoThreeBordersNoOutlier.csv};
% \addlegendentry{Volcano eruption without the outlier using heuristic initialization}
\addlegendentry{3 atoms, 11 members (outlier removed).}
\addplot[curve3] table [x=Timers, y=Loss, col sep=comma] {Appendix/PlotVolcanoFourChosenAtoms.csv};
% \addlegendentry{Volcanic erutipn using four chosen atoms}
\addlegendentry{4 atoms, 12 members.}
\end{axis}
\end{tikzpicture}
}\\
\subfloat{
\includegraphics[width=\linewidth]{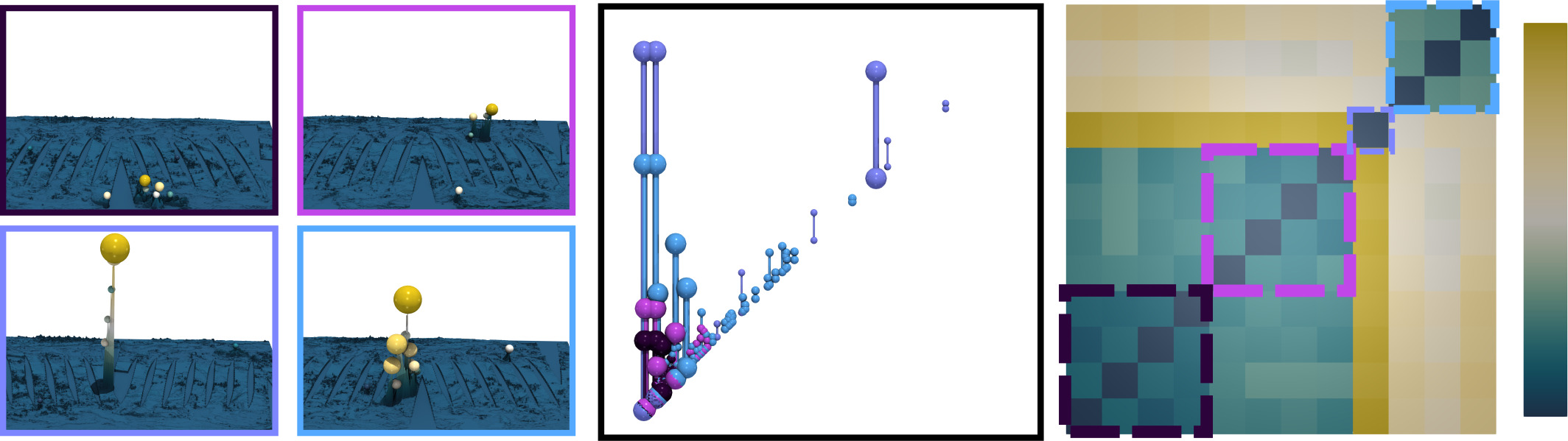}
}
\caption{\julienRevision{Evolution of the (normalized) energy
$\dictionaryEnergy$
along the optimization (top curves), with our multi-scale strategy, for the
\emph{Volcanic eruption ensemble}, for distinct initializations.
The ground-truth classification of this ensemble contains 3 classes
\cite{pont_vis21}, including one outlier (light purple entry in the bottom views,
from left
to right: terrain view of the data, aggregated birth/death space, distance
matrix).
A clear energy increase can be observed when considering the entire ensemble
(black curve), while a more characteristic oscillating behavior occurs when
discarding the outlier (green curve). When initializing the optimization with
4 atoms (1 per
class, plus 1 for the outlier), the optimization results in few
oscillations and a consistent energy decrease (yellow curve).}}

% However, one of these classes contains a clear outlier
% (light purple entry in the bottom views), corresponding to the peak of activity
% of the eruption. Specifically, this outlier includes
% features significantly more persistent than the other diagrams (see the light
% purple diagram in the aggregated birth/death space, bottom center, which is )}}

%
% Comparison of the energy optimization on the Volcanic Eruption data set depending on the initialization. From the matrix of distances, we can notice an outlier in the data set. The blue curve represents the energy optimized by our algorithm on the volcanic eruption data set with a dictionary of 3 atoms. The green one shows the evolution of the energy optimized on the volcanic eruption without the outlier. The last one shows what happens when we use the algorithm with four chosen atoms on the full data set.}}
\label{fig:volcano}
\end{figure}

%\bibliographystyle{abbrv}
%\bibliography{Appendix}

\end{document}